\documentclass[11pt]{article}
\usepackage[letterpaper,margin=1in]{geometry}
\usepackage[T1]{fontenc}
\usepackage{lmodern,microtype}
\usepackage{amsmath,amssymb,amsthm,mathtools,bm}
\usepackage{booktabs,tabularx,array,enumitem}
\usepackage{algorithm,algpseudocode}
\usepackage[numbers,sort&compress]{natbib}
\usepackage{xcolor,etoolbox}
\usepackage[colorlinks=true,linkcolor=red,citecolor=blue,urlcolor=black]{hyperref}

\DeclareRobustCommand{\reviewmark}[1]{%
  \textcolor{orange}{\normalfont\footnotesize\bfseries revise}\space}
\allowdisplaybreaks[2]
\numberwithin{equation}{section}
\newtheorem{theorem}{Theorem}
\newtheorem{definition}{Definition}
\newtheorem{lemma}[theorem]{Lemma}
\newtheorem{corollary}[theorem]{Corollary}
\theoremstyle{remark}
\newtheorem{remark}[theorem]{Remark}
\newcommand{\RR}{\mathbb R}
\newcommand{\EE}{\mathbb E}
\newcommand{\PP}{\mathbb P}
\newcommand{\1}{\mathbf 1}
\newcommand{\N}{\mathcal N}
\newcommand{\cE}{\mathcal E}
\newcommand{\cH}{\mathcal H}
\newcommand{\cX}{\mathcal X}
\newcommand{\MinReg}{\mathfrak R_T(d,K)}
\DeclareMathOperator{\tr}{tr}
\DeclareMathOperator{\op}{op}
\DeclareMathOperator{\KL}{KL}
\DeclareMathOperator{\TV}{TV}

\DeclareMathOperator*{\argmax}{arg\,max}

\newcommand{\ceil}[1]{\lceil #1\rceil}

\hypersetup{pdftitle={Nearly Minimax-Optimal Regret for Linear Contextual Bandits with Arbitrary Adaptive Action Sets},pdfauthor={Tianyuan Jin}}
\makeatletter
\renewcommand{\@maketitle}{%
  \newpage\null\vskip 1em
  \begin{center}%
    {\LARGE \@title\par}%
    \vskip 0.8em%
    {\large \@author\par}%
  \end{center}%
  \par\vskip 1.3em}
\makeatother
\title{Nearly Minimax-Optimal Regret for Linear Contextual Bandits with\\Arbitrary Adaptive Action Sets}
\author{Tianyuan Jin\\Data Science and Analytics Thrust\\The Hong Kong University of Science and Technology (Guangzhou)}
\date{}
\begin{document}
\maketitle

\begin{abstract}
We study stochastic linear contextual bandits with arbitrary action menus that may depend
on  the interaction history. We establish matching upper and lower bounds, up to logarithmic
factors. Let $d$ be the
dimension, $K$ be the menu size, and $T$ the time horizon. For $2\le K\le d$, we prove an upper bound $\widetilde O(K^{1/4}\sqrt{dT})$. When $T\ge d^2$, we further prove a lower bound $\Omega(K^{1/4}\sqrt{dT})$. Thus, for $T\ge d^2$ and $2\le K\le d$, the upper and lower bounds match up to logarithmic factors, and the polynomial dependence on $K$ is optimal. Compared with the $\widetilde O(\sqrt{dKT})$ bound of \citet{foster2020}, our upper bound improves the dependence on $K$ by a factor of $K^{1/4}$.

For $K\ge d$, we prove an upper bound $\widetilde O_{d,T}\left(\sqrt{dT}\min\{\sqrt d,(d\log K)^{1/4}\}\right)$
and a lower bound $\Omega\left(\sqrt{dT}\min\left\{\sqrt d,\left(\frac{d\log K}{\log(2d)}\right)^{1/4}\right\}\right)$. Here, $\widetilde O_{d,T}$ omits logarithmic factors only in $d,T$.
In particular, for polynomially large $K\ge d$, the upper and
lower bounds both scale as $d^{3/4}\sqrt T$ up to logarithmic factors, improving the standard $\widetilde O(d\sqrt T)$
bound \citep{abbasi2011} by a factor of $d^{1/4}$. As $K$ grows further, the regret smoothly
recovers the $d\sqrt T$ scale once $\log K$ reaches order $d$.

The authors developed the main ideas and guided the overall direction of the work. GPT-5.6 Sol Pro assisted in developing some of the proofs based on the ideas provided by the authors.
\end{abstract}

\section{Introduction}\label{sec:intro}
Contextual bandits provide a fundamental framework for sequential decision-making under partial
feedback,  with broad applications in adaptive routing, personalized recommendation, and mobile health interventions. In this paper, we study linear
contextual bandits. The problem proceeds in $T$ rounds. A parameter $\theta\in B_2^d$
is fixed before the
interaction. At each round $t\in[T]$, the learner observes a menu of actions, represented by
feature vectors $\mathcal A_t=(x_{t,1},\ldots,x_{t,K})$, $\|x_{t,i}\|_2\le1$. Here, we write the menu as containing exactly $K$
actions. Based on the current menu and past observations, the learner selects an action $A_t\in[K]$
and then observes  the reward $Y_t=x_{t,A_t}^\top\theta+\varepsilon_t$, $\varepsilon_t\sim\N(0,1)$, where the current noise is independent
of everything fixed before it is drawn. In linear bandits, the learner's goal is to minimize the difference between the total reward obtained by selecting an optimal action
$
i_t^*\in\argmax_{i\in[K]} x_{t,i}^\top\theta$
at each round $t$ and the total reward obtained by the learner.

There are two settings depending on how the menus $\{\mathcal A_t\}_{t\in[T]}$ are chosen. The first is the \emph{oblivious}
setting, where the entire sequence $\{\mathcal A_t\}_{t\in[T]}$
is chosen by the adversary at time $0$. SupLinUCB-style
methods achieve regret of order $\widetilde O\left(\sqrt{dT}\right)$ \citep{chu2011,auer2002,li2019}. On the lower-bound
side, \citet{chu2011} show a lower bound of $\Omega(\sqrt{dT})$ and, for $K\le 2^{d/2}$, \citet{li2019} further improve it to
$\Omega(\sqrt{dT\log K\log T})$.

The second is the \emph{adaptive} setting, where the menu $\mathcal A_t$ may depend on the fixed parameter $\theta$ and the complete past history
$H_{t-1}=(\mathcal A_1,A_1,Y_1,\ldots,\mathcal A_{t-1},A_{t-1},Y_{t-1})$ and $\theta$, that is $\mathcal A_{t}=\mathcal A_t(\theta, H_{t-1})$.  In this setting, self-normalized confidence
methods obtain $\widetilde O(d\sqrt T)$ regret \citep{abbasi2011}. SquareCB reduces contextual bandits
to online regression and, for a $d$-dimensional linear class, yields $\widetilde O(\sqrt{dKT})$ regret \citep{foster2020}. A natural open question is:
\begin{table}[t]
\centering
\caption{Comparison of regret bounds. Our lower bounds require $T\ge d^2$.  }\label{tab:comparison}
\small
\setlength{\tabcolsep}{8pt}
\renewcommand{\arraystretch}{1.16}
\begin{tabular}{@{}lcl@{}}
\toprule
Method & Menus & Regret Bound\\
\midrule
LinRel \citep{auer2002} and  
SupLinUCB \citep{chu2011} & Oblivious & $O\bigl(\sqrt{dT\log^3(KT)}\bigr)$\\
{VCL-SupLinUCB \citep{li2019}} & {Oblivious} & {$O(\sqrt{dT\log T\log K})\operatorname{poly}(\log\log(KT))$}\\
\midrule
Lin-TS \citep{agrawal2013,agrawal2014} & Adaptive & $O(d\log T\sqrt {T\log K})$\\
Lin-UCB/OFUL \citep{abbasi2011} & Adaptive & $O(d\sqrt T\log T)$\\
SquareCB \citep{foster2020} & Adaptive & $O(\sqrt{dKT\log T})$\\
Feel-Good TS \citep{zhang2022}  & Adaptive & $O(\min\{\sqrt{dKT\log T},d\sqrt{T\log T}\})$ \\
\midrule
Theorem \ref{thm:upper} ($2\le K\le d$) & Adaptive & $O(K^{1/4}\sqrt{dT}\log^{3/4}T)$\\
Lower Bound: Theorem \ref{thm:lower}($2\le K\le d$) & Adaptive & $\Omega(K^{1/4}\sqrt{dT})$\\
Theorem \ref{thm:large} ($K\ge d$) & Adaptive & $O(d^{3/4}\sqrt T\,[\min\{\log K,d\}]^{1/4}\log T)$\\
Lower Bound: Theorem \ref{thm:lower} ($K\ge d$) & Adaptive & $\Omega(d^{3/4}\sqrt T\,[\min\{\log K/\log(2d),d\}]^{1/4})$\\
\bottomrule
\end{tabular}
\end{table}

\begin{quote}
Is $\widetilde O(\min\{\sqrt{dKT},d\sqrt T\})$ worst-case optimal (within logarithm factors) when the menus
are chosen by an adaptive adversary? If not, what is the optimal regret bound?
\end{quote}
Our answer to the first question is no, and our answer to the second has two regimes.
\begin{enumerate}
\item \textbf{Small menus (Theorems~\ref{thm:lower} and~\ref{thm:upper}).} For $2\le K\le d\le T$, we prove an upper bound $\widetilde O(K^{1/4}\sqrt{dT})$. For $2\le K\le d$ and $T\ge d^2$, we also prove a lower bound $\Omega(K^{1/4}\sqrt{dT})$, showing that the polynomial dependence on $K$ in our upper bound cannot be improved. Compared with the previous $\widetilde O(\sqrt{dKT})$ bound of \citet{foster2020}, our upper bound improves the dependence on $K$ by a factor of $K^{1/4}$.

\item \textbf{Large menus (Theorems~\ref{thm:lower} and~\ref{thm:large}).} For $K\ge d$ and $T\ge d$, we prove an upper bound $\widetilde O_{d,T}\left(\sqrt{dT}\min\{\sqrt d,(d\log K)^{1/4}\}\right)$. For $K\ge d$ and $T\ge d^2$, we also prove a lower bound $\Omega\left(\sqrt{dT}\min\left\{\sqrt d,\left(\frac{d\log K}{\log(2d)}\right)^{1/4}\right\}\right)$. Thus, in the long-horizon regime, the upper and lower bounds have the same polynomial dependence on $d$, $T$, and $\log K$, up to logarithmic factors. Compared with the standard $\widetilde O(d\sqrt T)$ bound of \citet{abbasi2011}, our upper bound improves the dependence on $d$ by a factor of $(d/\log K)^{1/4}$ when $\log K\le d$. In particular, for polynomially large $K\ge d$, the upper bound is $d^{3/4}\sqrt T$ up to logarithmic factors, yielding a $d^{1/4}$ improvement over the standard $d\sqrt T$ rate. As $K$ increases further, the bound recovers the usual $d\sqrt T$ scale once $\log K$ is of order $d$.

\end{enumerate}
Table 1 compares our results with prior work.

\section{Problem Setting}\label{sec:model}
In our problem, a learner and the environment interact  over $T$ rounds. In each round $t=1,\ldots,T$:
\begin{enumerate}
\item The environment chooses a menu $\mathcal A_t=(x_{t,1},\ldots,x_{t,K})$ of $K$ actions, where $x_{t,i}\in
\RR^d$ and $\|x_{t,i}\|_2\le1$. In particular, $\mathcal A_t$ may depend on the fixed parameter $\theta$ and the past history $H_{t-1}:=
(\mathcal A_1,A_1,Y_1,\ldots,\mathcal A_{t-1},A_{t-1},Y_{t-1})$.
\item After observing $\mathcal A_t$, the learner chooses an action $A_t\in[K]$.
\item The learner observes $Y_t=x_{t,A_t}^\top\theta+\varepsilon_t$, where $\varepsilon_t\sim\N(0,1)$ is independent of everything fixed
before it is drawn, where $\theta\in B_2^d$ is a fixed and unknown target vector.
\end{enumerate}
Write $x_t:=x_{t,A_t}$ for the action actually played and $\mu_{t,i}:=x_{t,i}^\top\theta$ for the mean reward of action $i$.
Let $i_t^\star\in\argmax_{i\in[K]}\mu_{t,i}$. The cumulative pseudo-regret is defined as
\[
R_T:=\sum_{t=1}^T\left(\mu_{t,i_t^\star}-\mu_{t,A_t}\right).
\]
Our goal is to minimize $\mathbb E[R_T]$, where the expectation is over the learner's randomization, the environment's internal randomization, and the reward noise.
\section{High Level Ideas}\label{sec:outline}
The upper and lower bounds are driven by the same observation. An adaptive adversary can
exploit directions in which the learner is currently uncertain, but such attacks are not free. If the
adversary repeatedly uses completely new directions, it quickly runs out of dimensions. If it reuses
old directions, then the resulting observations reveal the directions in which the learner's predictor
is inaccurate. The lower bound shows how much of this information can be hidden, while the upper
bound shows how to exploit the information that necessarily becomes visible.
\subsection{{Lower Bounds}}
Our basic idea is to divide the horizon into consecutive blocks, each exploiting a different collection
of directions in which the learner's estimate may still be inaccurate. In a $d$-dimensional problem,
such prediction errors can persist along many different directions and might be changed over time.
A block uses some of these directions to construct a hard menu; after the observations in that block
reveal information about them, later blocks can use other directions in which substantial uncertainty
remains. Within each block, we establish two properties. First, we exploit the learner's uncertainty
in the selected directions to construct a menu containing a hidden good action, while ensuring that
the menu itself reveals only a small amount of information about $\theta$. This leaves enough uncertainty
in other directions to repeat the construction in later blocks. Second, we show that the rewards
within the block reveal the hidden action only slowly. Consequently, the learner misses the hidden
action on a constant fraction of the rounds and the block incurs substantial regret.

At the beginning of a block, the learner still has uncertainty about the fixed parameter $\theta$. Let $h$
and $\Sigma$ be its current posterior mean and covariance, i.e., conditional on history $H$, $\theta\sim\N(h,\Sigma)$.
We choose $D=\lfloor d/2\rfloor-1$ directions in which the posterior variance is of order $\sigma^2$.
Let $U\in\RR^{d\times D}$
contain these directions and write $S:=U^\top\Sigma U$. The normalized estimation error in this subspace
is $z:=S^{-1/2}U^\top(\theta-h)$, so $z\sim\N(0,I_D)$, where $I_D$ denotes the $D$-dimensional identity matrix.
We now create $K$ candidate actions and hide one special index $J\sim\operatorname{Unif}([K])$. Draw independent
$G_1,\ldots,G_K\sim\N(0,I_D)$. For $i\ne J$, the  action is generated from $G_i$, while for the hidden index
we use
\[
\rho z+\sqrt{1-\rho^2}G_J,
\]
where $\rho$ controls the level of correlation with the unknown parameter. The hidden action looks exactly like every other action. Since both $z$
and $G_J$ are standard Gaussian, $\rho z+\sqrt{1-\rho^2}G_J$ is also standard Gaussian. Hence the distribution
of the whole menu is the same for every value of $J$, and observing the menu does not reveal which
index is hidden.

Nevertheless, the hidden action is better because it is correlated with the
estimation error $z$. This creates a reward gap of order $\Delta\asymp\sigma\rho\sqrt D$.  Thus a larger $\rho$ creates a larger reward gap, but
also makes the hidden index $J$ easier to identify from the observed rewards.

\paragraph{Rewards reveal the hidden action slowly.}
Whenever the learner fails to pull $J$, it incurs regret
of order $\Delta$. The question is how quickly the learner can identify $J$ from the observed
rewards.

The key point is that the average KL divergence to a common reference experiment contributed by one reward is only $O(\Delta^2/K)$. To explain this, introduce a common reference experiment in which the planted shift associated with the hidden index is removed, while the menu, learner policy, residual uncertainty, and reward noise are kept unchanged. Conditional on the menu and $J$, let $\mu_J$ denote the planted shift in the conditional mean of $U^\top\theta$. If the learner pulls action $i$, its reward mean under hidden index $J$ differs from that in the reference experiment by
$
(U^\top x_i)^\top\mu_J.
$
Since the reward noise is Gaussian, the corresponding KL divergence is proportional to
$
\bigl((U^\top x_i)^\top\mu_J\bigr)^2.
$ The random-menu geometry ensures that, when $J$ is uniform,
$
\EE_J\!\left[\bigl((U^\top x_i)^\top\mu_J\bigr)^2\right]
=
O(\Delta^2/K).
$

Hence after $n$ rounds the average KL divergence to the reference experiment is only $O(n\Delta^2/K)$. As
long as $n\Delta^2/K=O(1)$, the learner cannot reliably identify the hidden action and therefore misses
it on a constant fraction of the rounds. The block consequently incurs $\Omega(n\Delta)$ regret and can remain
hard for $n\asymp K/\Delta^2\asymp K/(\sigma^2D\rho^2)$ rounds.

\paragraph{Controlling information across blocks.}
We next determine how many hard blocks can be generated. Fix the variance scale
$\sigma^2$ of the initial Gaussian prior; this quantity remains fixed throughout
the construction. What changes from block to block is the posterior covariance.
Writing $\Sigma_b$ for the covariance at the beginning of block $b$, we have
$\Sigma_1=\sigma^2 I_d$ and $\Sigma_{b+1}\preceq\Sigma_b$. Our block construction
remains valid as long as
$\operatorname{tr}(\Sigma_b^{-1})<2d/\sigma^2$.

Each block consumes this posterior-precision budget in two ways. First, the menu
itself is correlated with the unknown parameter. Using correlation strength
$\rho$ in a $D$-dimensional uncertain subspace increases the posterior precision
by $O(D\rho^2/\sigma^2)$ per block. Second, the rewards observed within a block
also reveal information about $\theta$: a block of length $n$ contributes at
most $\sum_s \|x_s\|_2^2\le n$ additional trace precision. Across the whole
horizon, the latter contribution is therefore at most $T$. With the sufficiently
small constant in our choice $\sigma^2\asymp d/T$, this uses only a constant
fraction of the available budget $d/\sigma^2$ with $\sigma^2\asymp d/T$. Hence the number of reusable hard
blocks is essentially determined by the information leaked through the menus.

Thus, after $B$ blocks, the menu contribution is
$O(BD\rho^2/\sigma^2)$. Keeping the posterior uncertainty large enough for
another hard block requires
$BD\rho^2/\sigma^2\lesssim d/\sigma^2$, and therefore allows about
$B\asymp d/(D\rho^2)$ blocks. By the reward-information calculation above, each
block can remain hard for
$n\asymp K/(\sigma^2D\rho^2)$ rounds. Hence the total number of hard rounds is
of order
$dK/(\sigma^2D^2\rho^4)$. Substituting $\sigma^2\asymp d/T$, this becomes
$TK/(D^2\rho^4)$. We therefore choose $D\rho^2\asymp\sqrt K$, so that the hard
blocks occupy a constant fraction of the horizon. With this choice, the reward
gap satisfies
$\Delta\asymp\sigma\rho\sqrt D\asymp K^{1/4}\sqrt{d/T}$. Since the learner
incurs regret of order $\Delta$ on a constant fraction of the $T$ rounds, we
obtain
$R_T=\Omega(K^{1/4}\sqrt{dT})$.

\paragraph{Large menus.}
The small-menu construction hides one good action among $K$ candidate actions.
To obtain a lower bound that continues to improve when $K\gg d$, we use a product construction.

We divide the available directions into $m$ orthogonal groups. In each group $\ell$, we construct $k$
candidate vectors $x_{\ell,1},\ldots,x_{\ell,k}$, with one hidden good candidate. An action is obtained by
choosing one candidate from each group and summing the $m$ selected vectors, scaled by $1/\sqrt m$.
Thus each index tuple $(i_1,\ldots,i_m)\in[k]^m$ defines one action,
giving $k^m$ distinct actions, with the construction requiring
$d\gtrsim mk$. Choosing $m$ and $k$ so that $mk\lesssim d$ and $k^m\le K$
allows the construction to exploit large menus.

The $1/\sqrt m$ scaling keeps every action inside the unit ball. It also reduces the reward contribution
of each group by $1/\sqrt m$. Consequently, the information about the hidden index in any one group is
reduced by a factor $1/m$, while the regret contributions from the $m$ groups add. Together, these
effects yield an overall $\sqrt m$ amplification in the regret. Optimizing $m$ and $k$ under $mk\lesssim d$ and
$k^m\le K$ gives the large-menu lower bound, with the dependence on $K$ entering through $\log K$.

\subsection{Upper Bound Part}
The algorithm maintains $V_t=I+\sum_{s<t}x_sx_s^\top$ and $w_{t,i}=\sqrt{x_{t,i}^\top V_t^{-1}x_{t,i}}$. For small menus, we use
EXP4-IX \citep{neu2015}, which competes with the best of $N$ experts at cost $\widetilde O(\sqrt{KT\log N})$. We construct a finite expert family with
small logarithmic size that contains a good predictor.

To define the comparator expert, consider an auxiliary repair process that knows $\theta$. Whenever some action has prediction
error larger than $\alpha w_{t,i}$, the auxiliary process repairs its predictor in the corresponding direction. Each repair
yields at least $\alpha^2$ progress, while the total linear information is controlled, as in OFUL \citep{abbasi2011}, by $\log\det V_{T+1}\le d\log(1+T/d)$. Hence, with high probability, the auxiliary process makes
at most $M=\widetilde O(d/\alpha^2)$ repairs. Each repair is described by a triple $(t,i,s)\in[T]\times[K]\times\{-1,1\}$,
so considering all repair histories of length at most $M$ gives $N\le(2KT+1)^M$ and therefore
$\log N=\widetilde O(d/\alpha^2)$. One of these experts exactly follows the oracle and predicts every displayed
action within $\alpha w_{t,i}$.

The remaining difficulty is that EXP4-IX does not necessarily play the action recommended
by the good expert. A direct comparison would therefore leave an uncertainty term for an action
that may never be sampled. We use a discounted loss to move this uncertainty to the action
actually played. For Gaussian rewards, if action $i$ is played, then $Y_t\sim\N(\mu_{t,i},1)$, and hence
$\Pr(Y_t\le0\mid A_t=i)=\Phi(-\mu_{t,i})$, where \(\Phi\) denote the cdf  of \(N(0,1)\). We therefore use the bounded loss $\ell_t(i):=\Phi(-\mu_{t,i})$, which
admits the unbiased bounded observation $\1_{\{Y_t\le0\}}$ when action $i$ is played. Let $\beta_{t,i}=\min\{1,\alpha w_{t,i}\}$ and
define $\bar\ell_t(i)=\frac{\ell_t(i)+1-\beta_{t,i}}2$. Let $a_t^\star$ be the action recommended by the good expert and let $i_t^\star$ be an
optimal action. Since the good expert is accurate and optimistic, its reward gap is at most $2\alpha w_{t,a_t^\star}$,
which implies $\ell_t(a_t^\star)-\ell_t(i_t^\star)\le\beta_{t,a_t^\star}$. For the action $A_t$ actually played by the learner, a direct
rearrangement gives $\ell_t(A_t)-\ell_t(i_t^\star)\le2\left(\bar\ell_t(A_t)-\bar\ell_t(a_t^\star)\right)+\beta_{t,A_t}$. The first term is controlled by
EXP4-IX, while the second depends only on the uncertainty of the action actually played and can
therefore be summed using the standard elliptical-potential bound.

Since $\log N=\widetilde O(d/\alpha^2)$, the master cost is $\widetilde O(\sqrt{dKT}/\alpha)$, while the elliptical-potential bound gives $\sum_t\beta_{t,A_t}=\widetilde O(\alpha\sqrt{dT})$. Hence $\EE[R_T]=\widetilde O(\sqrt{dT}[\sqrt K/\alpha+\alpha])$, and choosing $\alpha\asymp K^{1/4}$ gives $\EE[R_T]=\widetilde O(K^{1/4}\sqrt{dT})$.

For large menus, we keep the repair construction and the same transfer of uncertainty to played actions, but replace EXP4-IX by a geometric master with a linear surrogate reward. Geometric exploration replaces the master's explicit $K$ dependence by $\min\{K,d+1\}$, while the expert-family size still depends logarithmically on $K$.

\section{Lower Bounds}\label{sec:lower-statements}
Let $\mathcal X_K:=(B_2^d)^K$ be the set of all menus containing $K$ unit-ball actions. Let $\mathcal Q$ be the class of all randomized nonanticipating menu mechanisms $Q=(Q_t)_{t=1}^T$, where $
Q_t(\,\cdot\mid \theta,H_{t-1})
$
is a probability distribution over $\mathcal X_K$. Thus the current menu may depend on the fixed parameter $\theta$, the past history, and the environment's internal randomness, but not on the learner's current action or the current reward noise. 
For any policy $\pi$ and a menu mechanism $Q$, define the minimax optimal regret as follows.
\begin{equation}\label{eq:minimax}
\MinReg=\inf_\pi\sup_{\theta\in B_2^d}\sup_Q\EE_{\theta,Q,\pi}R_T.
\end{equation}
\begin{theorem}[Minimax lower bounds]\label{thm:lower}
Suppose $d,K\ge2$ and $T\ge d^2$. There is a universal constant $c>0$ such that
\begin{align}
\MinReg&\ge   {c}K^{1/4}\sqrt{dT},\qquad 2\le K\le d,\label{eq:small-lower}\\
\MinReg&\ge   {c}\sqrt{dT}\min\left\{\sqrt d,\left(\frac{d\log K}{\log(2d)}\right)^{1/4}\right\}{,\qquad K\ge d.}\label{eq:large-lower}
\end{align}
\end{theorem}
\begin{corollary}[Oblivious and adaptive menus]\label{cor:adaptivity-gap}
For $K=d$ and $T\ge d^2$, adaptive menus have minimax regret
$\Omega(d^{3/4}\sqrt T)$, whereas oblivious menus admit
$\widetilde O(\sqrt{dT})$ regret. Thus adaptivity can increase the minimax
regret by a factor of $d^{1/4}$, up to logarithmic factors.
\end{corollary}

The detailed proof of Theorem~\ref{thm:lower} is given in
Section~\ref{sec:lower-analysis}, with the supporting lemmas proved in
Appendix~\ref{app:lower}. Beyond the separation in
Corollary~\ref{cor:adaptivity-gap}, Theorem~\ref{thm:lower} gives a finer
dependence on the menu size. For $K\le d$, the lower bound scales as
$\Omega(K^{1/4}\sqrt{dT})$. For $K\ge d$, it grows with
$(\log K)^{1/4}$ until reaching the $d\sqrt T$ scale. In particular, for
polynomially large $K\ge d$, the lower bound is
$d^{3/4}\sqrt T$ up to logarithmic factors.

\section{Algorithms}\label{sec:algorithms}
Both algorithms use a shared ridge regression together with a finite family of repair sequences.
Each repair sequence modifies the shared estimate through its own offset. The two algorithms differ
 {in the master used to combine the expert recommendations and the 
feedback used to update their weights.} We first present the
small-menu algorithm for $K\le d$, and then replace its master with a geometric one for large menus.

\subsection{The Algorithm for Small Menus}\label{sec:small-algorithm}
This subsection presents Repair-IX, whose pseudo-code is given in Algorithm~\ref{alg:upper}. 
Both algorithms enumerate the full repair family and are not polynomial-time.
For $\alpha>0$, set
$M:=\ceil{(1+d\log(1+T/d)+4\log T)/\alpha^2}$ and define
\begin{equation}\label{eq:repair-family}
\cE_M:=\left\{e=((\tau_j,i_j,s_j))_{j=1}^r:0\le r\le M,\ 1\le\tau_1\le\cdots\le\tau_r\le T,\ i_j\in[K],\ s_j\in\{-1,1\}\right\}.
\end{equation}
Thus each $e\in\cE_M$ is a repair sequence containing at most $M$ repairs, and $|\cE_M|\le(2KT+1)^M$. All repair sequences share the ridge statistics $V_t:=I_d+\sum_{s<t}x_sx_s^\top$ and $b_t:=\sum_{s<t}x_sY_s$. For each
displayed action, let $w_{t,i}:=\sqrt{x_{t,i}^\top V_t^{-1}x_{t,i}}$ and $\beta_{t,i}:=\min\{1,\alpha w_{t,i}\}$. Each repair sequence $e\in\cE_M$
 {maintains an offset $c_e$, initialized at zero. This offset records that
expert's repairs to the shared estimate; it does not use a separate data set.} If $e$ contains a repair $(t,i,s)$, it applies
\begin{equation}\label{eq:repair-update}
 {c_e\leftarrow c_e+\frac{\alpha s}{w_{t,i}}x_{t,i}\quad\text{if }w_{t,i}>0.}
\end{equation}
 {If $w_{t,i}=0$, the repair leaves $c_e$ unchanged. After applying all repairs specified by $e$ at round $t$, it forms} $v_{t,e}:=V_t^{-1}(b_t+c_e)$ and recommends
\begin{equation}\label{eq:recommendation}
a_{t,e}\in\argmax_{i\in[K]}\left\{x_{t,i}^\top v_{t,e}+\alpha w_{t,i}\right\}.
\end{equation}
Each repair sequence is treated as an expert with weight $W_e$, initialized at one. The weights define
\begin{equation}\label{eq:ix-probabilities}
p_{t,i}:=\frac{\sum_{e\in\cE_M}W_e\1_{\{a_{t,e}=i\}}}{\sum_{e\in\cE_M}W_e}.
\end{equation}
After sampling $A_t\sim p_t$ and observing $Y_t$, define $Z_t:=(\1_{\{Y_t\le0\}}+1-\beta_{t,A_t})/2$ and use
\begin{equation}\label{eq:ix-update}
\widehat\ell_{t,e}:=\frac{Z_t\1_{\{a_{t,e}=A_t\}}}{p_{t,A_t}+\eta},\qquad
W_e\leftarrow W_e\exp(-\eta\widehat\ell_{t,e}).
\end{equation}
  
\begin{algorithm}[H]
\caption{Repair-IX: shared ridge regression with repair sequences}\label{alg:upper}
\begin{algorithmic}[1]
\Require $\alpha>0$ and $\eta>0$
\State Let $\cE_M$ be the family of schedules defined in  \eqref{eq:repair-family}
\State Initialize $V_1\leftarrow I_d$, $b_1\leftarrow0$, and $c_e\leftarrow0$, $W_e\leftarrow1$ for every $e\in\cE_M$
\For{$t=1,\ldots,T$}
\State Observe $x_{t,1},\ldots,x_{t,K}$ and compute $w_{t,i}$ and $\beta_{t,i}$ for every $i\in[K]$
\State For every $e\in\cE_M$, apply its repairs at round $t$ using \eqref{eq:repair-update}
\State For every $e\in\cE_M$, compute its recommendation using \eqref{eq:recommendation}
\State Compute $p_t$ using \eqref{eq:ix-probabilities} and sample $A_t\sim p_t$
\State Observe $Y_t$, compute $Z_t$, and update every $W_e$ using \eqref{eq:ix-update}
\State $V_{t+1}\leftarrow V_t+x_{t,A_t}x_{t,A_t}^\top$, $b_{t+1}\leftarrow b_t+x_{t,A_t}Y_t$
\EndFor
\end{algorithmic}
\end{algorithm}

\begin{theorem}[Upper bound for small menus]\label{thm:upper}
For $2\le K\le d\le T$, run Algorithm~\ref{alg:upper} with
$\alpha=[K\log(2KT+1)]^{1/4}$ and
$\eta=\sqrt{M\log(2KT+1)/(KT)}$.
Then, for a universal constant $C>0$,
\begin{equation}\label{eq:upper-main}
\EE[R_T]\le   {C}\,[K\log(2KT+1)]^{1/4}\sqrt{dT\log(1+T/d)}.
\end{equation}
\end{theorem}
The proof of Theorem~\ref{thm:upper} is in
Appendix~\ref{sec:upper-analysis}. The supporting lemmas are proved in
Appendix~\ref{app:upper}. \textbf{Comparison with previous work:}
For $2\le K\le d$, the best previous bound for adaptive
menus is $\widetilde O(\sqrt{dKT})$, obtained by SquareCB
\citep{foster2020}. Theorem~\ref{thm:upper} improves this bound by a factor
of $K^{1/4}$. Together with Theorem~\ref{thm:lower}, this gives matching
upper and lower bounds up to logarithmic factors when $T\ge d^2$.
\subsection{Algorithm for Large Menus}\label{sec:large-algorithm}
\begin{algorithm}[H]
\caption{Repair-Geo: repair sequences with a geometric master}
\label{alg:geo}
\begin{algorithmic}[1]
\Require $\alpha>0$ and
$0<\eta\le1/(   {16}\min\{K,d+1\})$
\State Let $\cE_M$ be defined as in \eqref{eq:repair-family}
\State Set $\gamma\leftarrow   {8}\min\{K,d+1\}\eta$
\State Initialize $V_1\leftarrow I_d$, $b_1\leftarrow0$, and
$c_e\leftarrow0$, $W_e\leftarrow1$ for every $e\in\cE_M$
\For{$t=1,\ldots,T$}
    \State Observe $x_{t,1},\ldots,x_{t,K}$ and compute
    $w_{t,i}$ and $\beta_{t,i}$ for every $i\in[K]$
    \State For every $e\in\cE_M$, apply all repairs of $e$ scheduled at
    round $t$ using \eqref{eq:repair-update}, and compute $a_{t,e}$
    using \eqref{eq:recommendation}
    \State Form the lifted features $z_{t,1},\ldots,z_{t,K}$
    \State Compute $p_t^0$ and    {compute $\nu_t$ by Lemma~\ref{lem:design}}
    \State Set
    $p_t\leftarrow(1-\gamma)p_t^0+\gamma\nu_t$ and
    $\Gamma_t\leftarrow\sum_i p_{t,i}z_{t,i}z_{t,i}^\top$
    \State Sample $A_t\sim p_t$
    \State Observe $Y_t$ and set
    $Z_t\leftarrow(Y_t+2\beta_{t,A_t}-1)/2$
    \State For every $e\in\cE_M$, compute $\widehat g_{t,e}$ and
    $h_{t,e}$, and update
    $W_e\leftarrow
    W_e\exp\{\eta(\widehat g_{t,e}+2\eta h_{t,e})\}$
    \State
    $V_{t+1}\leftarrow V_t+x_{t,A_t}x_{t,A_t}^\top$ and
    $b_{t+1}\leftarrow b_t+x_{t,A_t}Y_t$
\EndFor
\end{algorithmic}
\end{algorithm}
Repair-Geo keeps the repair family $\cE_M$, ridge statistics, repairs, and
recommendations $a_{t,e}$ from Repair-IX. Its master uses the shared linear
structure of the rewards, rather than treating the $K$ actions as unrelated
arms. The geometric estimator and leverage correction follow the approach
of \citet{zimmert2022}, adapted here to changing action features and repair
experts chosen for comparison after the history is known.

\paragraph{The master.}
Choose $0<\eta\le1/(16\min\{K,d+1\})$ and set
$\gamma:=8\min\{K,d+1\}\eta$.
The sign-based loss used by Repair-IX is not linear in the action features,
so here we use a linear surrogate reward.
For action $i$, define
$z_{t,i}:=(x_{t,i}^\top,2\beta_{t,i}-1)^\top$;
if $A_t$ is played, set
$Z_t:=(Y_t+2\beta_{t,A_t}-1)/2$.
With $\vartheta:=\frac12(\theta^\top,1)^\top$, its conditional mean is
\[
g_t(i):=\EE_t[Z_t\mid A_t=i]
=z_{t,i}^\top\vartheta
=\frac{\mu_{t,i}+2\beta_{t,i}-1}{2}.
\]
The term $2\beta_{t,i}$ compensates for the possible reward gap of the good
repair expert, so that the remaining uncertainty can be charged to the
actions actually played.

The expert weights first define
\[
p_{t,i}^0:=
\frac{\sum_{e\in\cE_M}W_e\1_{\{a_{t,e}=i\}}}
{\sum_{e\in\cE_M}W_e}.
\]
This distribution may concentrate on only a few directions of the lifted
feature space. We therefore mix it with an approximate exploration design.
Let $\nu_t\in\Delta_K$ be the approximate design computed by the finite
procedure in Lemma~\ref{lem:design}, where
$\Delta_K:=\{\nu\in\RR_+^K:\sum_i\nu_i=1\}$.
Its largest leverage is at most twice the dimension of the current feature
span. This span is nonzero: if $x_{t,i}=0$, then $\beta_{t,i}=0$ and the last
coordinate of $z_{t,i}$ is $-1$. We mix the design with the expert distribution:
\begin{equation}\label{eq:geo-distribution}
p_t
:=
\underbrace{(1-\gamma)p_t^0}_{\text{follow the repair experts}}
+
\underbrace{\gamma\nu_t}_{\text{geometric exploration}},
\qquad
\Gamma_t
:=
\sum_{i=1}^K p_{t,i}z_{t,i}z_{t,i}^\top.
\end{equation}
The learner samples $A_t\sim p_t$. The matrix $\Gamma_t$ records how well
the resulting sampling distribution covers the different directions in the
current lifted feature space.

After observing $Z_t$, Repair-Geo uses the linear structure to evaluate all
repair experts from this single observation. Write $\Gamma_t^\dagger$ for
the inverse of $\Gamma_t$ on
$\operatorname{span}\{z_{t,1},\ldots,z_{t,K}\}$, extended by zero on its
orthogonal complement. For expert $e$, define
$\widehat g_{t,e}:=
z_{t,a_{t,e}}^\top\Gamma_t^\dagger z_{t,A_t}Z_t$.
 {For each fixed expert $e$, this estimate is conditionally unbiased:
$\EE_t[\widehat g_{t,e}]=g_t(a_{t,e})$.
It uses the single observed reward to score even experts whose recommended
actions were not played.} The factor $\Gamma_t^\dagger$ plays the same role as an
importance-weighting correction, but uses the geometry shared by all actions
rather than treating their sampling probabilities independently.

We also define
$h_{t,e}:=z_{t,a_{t,e}}^\top\Gamma_t^\dagger z_{t,a_{t,e}}$.
This leverage score measures how well the direction recommended by expert
$e$ is covered by the current sampling distribution. A poorly explored
direction has a larger leverage score, which also controls the second moment
of its geometric reward estimate. The expert weights are updated by
\begin{equation}\label{eq:geo-update}
W_e
\leftarrow
W_e
\exp{\eta(\widehat g_{t,e}+2\eta h_{t,e})}.
\end{equation}
The positive correction $2\eta h_{t,e}$ controls cumulative underestimation
of the repair expert chosen for comparison after the trajectory is observed.
We use Repair-Geo when $\log K\le d$ and
$M\log(2KT+1)\le T/(16\min\{K,d+1\})$.
Otherwise we use OFUL with the parameters of Lemma~\ref{lem:oful} and
its rule \eqref{eq:oful-rule}. The following theorem bounds this combined rule.

\begin{theorem}[Upper bound for large menus]
\label{thm:large}
For $K\ge d\ge2$ and $T\ge d$, let
$\alpha=[\min\{K,d+1\}\log(2KT+1)]^{1/4}$.
If $\log K\le d$ and
$M\log(2KT+1)
\le \frac{T}{16\min\{K,d+1\}}$,
run Algorithm~\ref{alg:geo} with
$
\eta=
\sqrt{
\frac{M\log(2KT+1)}
{16\min\{K,d+1\}T}
}$.
Otherwise, run OFUL with $\lambda=1$ and $\delta=T^{-2}$.
Then, there exists a universal constant $C>0$,
\begin{equation}\label{eq:large-upper}
\EE[R_T]
\le
C\log(2dT)\sqrt{dT}
\min\left\{
\sqrt d,\,
(d\log K)^{1/4}
\right\}.
\end{equation}
\end{theorem}

The proof of Theorem~\ref{thm:large} is given in
Appendix~\ref{sec:proof-large}. \textbf{Comparison with previous work:}
For $K\ge d$, the best previous general bound for
adaptive menus is the standard $\widetilde O(d\sqrt T)$ confidence-based bound \citep{abbasi2011}.
Theorem~\ref{thm:large} improves this to $\widetilde O(d^{3/4}\sqrt T(\log K)^{1/4})$, corresponding to an improvement by a factor of
$(d/\log K)^{1/4}$. In particular, for polynomially large $K$, the improvement is $d^{1/4}$ up to logarithmic
factors. Together with the lower bound $\Omega\left(\sqrt{dT}(d\log K/\log(2d))^{1/4}\right)$
from Theorem~\ref{thm:lower}, our upper bound is optimal up to logarithmic factors in $d$ and $T$.

\section{Conclusion and Open Problems}\label{sec:conclusion}
We studied linear contextual bandits with adaptively chosen action menus and established matching
upper and lower bounds on the worst-case regret up to logarithmic factors. For $2\le K\le d$, the
optimal dependence is $K^{1/4}\sqrt{dT}$ up to logarithmic factors. For $d\le K$, the dependence on the menu
size grows only as $(\log K)^{1/4}$; in particular, for polynomially large $K$, the optimal regret is $d^{3/4}\sqrt T$
up to logarithmic factors.

The main open problem is computational efficiency. Both Repair-IX and Repair-Geo maintain
an exponentially large family of repair sequences. It remains open whether one can obtain the same
regret guarantees with a polynomial-time algorithm, while retaining good empirical performance on
standard linear-bandit instances.
\bibliographystyle{plainnat}
\bibliography{references}
\clearpage
\appendix
\section{Related Work}\label{sec:related}

\paragraph{Linear bandits with adaptive features.}
Early work on linear contextual bandits mainly considered action features
that are fixed before the interaction. \citet{chu2011} gave SupLinUCB with
regret $O(\sqrt{dT\log^3(KT\log T)})$ and proved a matching
$\Omega(\sqrt{dT})$ lower bound up to logarithmic factors.
For $K\le 2^{d/2}$, \citet{li2019} sharpened this picture by proving a lower
bound of $\Omega(\sqrt{dT\log T\log K})$ and a VCL-SupLinUCB upper bound
matching it up to iterated logarithmic factors.
These results rely on the action features being chosen independently of the
past reward noise.

When the features may depend on past observations, general confidence-based
methods such as OFUL \citep{abbasi2011} give
$\widetilde O(d\sqrt T)$ regret. Linear Thompson sampling also allows
adaptive contexts; the revised analysis of
\citet{agrawal2013,agrawal2014} gives
$O(d\log T\sqrt{T\log K})$ regret.
\citet[Section~2.3]{foster2020} highlight the gap between fixed and adaptive
features and ask whether the additional dependence on $K$ is necessary.
Our results answer this question by giving the minimax dependence on the menu
size $K$ up to logarithmic factors.

\paragraph{Reductions to supervised learning.}
Another line of work reduces contextual bandits to supervised learning.
ILOVETOCONBANDITS \citep{agarwal2014} considers a finite policy class of
size $N$ under i.i.d.\ context--reward pairs and achieves the optimal
$O(\sqrt{KT\log N})$ regret up to logarithmic factors using a
cost-sensitive classification oracle.
Under realizability, \citet{simchilevi2022} give an optimal reduction to
offline regression using only $O(\log T)$ oracle calls, or
$O(\log\log T)$ calls when the horizon is known in advance.
Both approaches rely on independently drawn contexts.

SquareCB instead reduces contextual bandits to online regression and allows
the contexts to depend on the past \citep{foster2020}. If the online
regression oracle has regret $B_T$, SquareCB gives regret of order
$\widetilde O(\sqrt{KTB_T})$; for a $d$-dimensional linear class this becomes
$\widetilde O(\sqrt{dKT})$.
This makes SquareCB directly applicable to our adaptive-menu setting, but
leaves a $\sqrt K$ dependence. Our bounds use the linear structure more
directly and improve this dependence to $K^{1/4}$ for small menus and to
$(\log K)^{1/4}$ for large menus, up to logarithmic factors.

\paragraph{Combining experts.}
Our upper bounds also build on existing expert and adversarial-bandit
methods. Repair-IX uses the implicit-exploration estimator of
\citet{neu2015}; in particular, EXP4-IX achieves
$O(\sqrt{KT\log N})$ regret against $N$ experts up to lower-order and
high-probability terms, without explicit uniform exploration.
Repair-Geo uses a geometric estimator and leverage correction following the
approach of \citet{zimmert2022}. For a fixed finite action set, their
exponential-weights method achieves
$O(\sqrt{dT\log(KT)})$ regret against an adaptive adversary.

\section{Proof of the Lower Bounds}\label{sec:lower-analysis}
We first focus on small $K$. Our proofs require the following lemmas.
\subsection{Key Technical Lemmas}
For the proof we use a slightly easier auxiliary experiment.
The unnormalized menu vectors are revealed at the beginning of each block,
and its hidden labels are revealed after its last reward. The independent
Gaussian seeds used to form those vectors are not revealed. Write $\cH$
for the resulting augmented history before the next block. Conditioning on
$\cH$ preserves a Gaussian posterior. This history is used only in the
lower-bound analysis; it is not the learner's visible history $H_{t-1}$ in
the original model. Remark~\ref{rem:disclosures} explains how to remove the
disclosures.

\begin{lemma}\label{lem:blocks}

Let $m,n,k$ be positive integers and let $\sigma>0$, with $k\ge2^{40}$,
$16mk\le d$, and $k^m\le K$. In the auxiliary experiment,
the unnormalized menu vectors are revealed before a block and its hidden
labels after the block. Suppose that, at its start, the augmented history
$\cH$ satisfies $\theta\mid\cH\sim\N(h,\Sigma)$ with

\begin{equation}\label{eq:available-uncertainty}
0\prec\Sigma\preceq\sigma^2I_d,\qquad\tr(\Sigma^{-1})<\frac{2d}{\sigma^2}.
\end{equation}
If $n\sigma^2/(m\sqrt k)\le1/64$, then there exists an $n$-round block construction using the same menu of
exactly $K$ distinct unit-ball actions throughout the block such that any learner satisfies
\begin{equation}\label{eq:reusable-regret}
 {\EE[R_{\mathrm{block}}\mid\cH]}\ge\frac{n\sigma\sqrt m\,k^{1/4}}{256}.
\end{equation}
After the $n$ rewards and the delayed hidden labels are revealed, the posterior is again Gaussian, with
covariance $\Sigma_+\preceq\Sigma$ satisfying
\begin{equation}\label{eq:reusable-precision}
\tr(\Sigma_+^{-1})\le\tr(\Sigma^{-1})+\frac{8m\sqrt k}{\sigma^2}+\sum_{s\text{ in this block}}\|x_s\|_2^2.
\end{equation}
\end{lemma}
For $m=1$ the block contains one hidden choice among $k$ actions. For general $m$, it combines $m$
such choices into $k^m$ product actions. The factor $\sqrt m$ in \eqref{eq:reusable-regret} comes from this product construction.
The proof is in Appendix~\ref{app:lower}.

\begin{lemma}[Gaussian truncation]\label{lem:truncation}
Let $T\ge d^2$, $\sigma^2=d/(4096T)$, and $\theta\sim\N(0,\sigma^2I_d)$. Then,
\begin{equation}\label{eq:tail-regret}
\EE[R_T\1_{\{\|\theta\|_2>1\}}]\le\frac{T\sigma}{1024}.
\end{equation}
\end{lemma}
The Gaussian prior may place a small amount of mass outside the unit ball. Lemma~\ref{lem:truncation} shows
that the regret contributed by this event is negligible. We may therefore condition the prior on
$\{\|\theta\|_2\le1\}$ and lose only the small term in \eqref{eq:tail-regret}. The details of the proof are in Appendix~\ref{app:truncation}.

\begin{lemma}\label{lem:product}
If $k\ge2^{40}$, $16mk\le d$, $k^m\le K$, and $T\ge d^2$, then
\begin{equation}\label{eq:product-lower}
\MinReg\ge2^{-16}\sqrt m\,k^{1/4}\sqrt{dT}.
\end{equation}
\end{lemma}
\begin{lemma}[Oblivious baseline; cf.\ \citet{chu2011}]\label{lem:baseline}
For $d,K\ge2$ and $T\ge d$, there is an oblivious menu
construction with $K$ distinct actions per round such that
\begin{equation}\label{eq:baseline-lower}
\MinReg\ge c_0\sqrt{dT}
\end{equation}
for a universal constant $c_0>0$.
\end{lemma}
\subsection{Proof of Theorem~\ref{thm:lower}}
\begin{proof}
\textbf{Small menus.}
Suppose $2\le K\le d$.
If $K\ge2^{45}$, take $m=1$ and
$k=\min\{K,\lfloor d/16\rfloor\}$.
Then $k\ge K/32\ge2^{40}$, $16k\le d$, and $k\le K$.
Lemma~\ref{lem:product} gives
\[
\MinReg
\ge
2^{-16}k^{1/4}\sqrt{dT}
\ge
2^{-18}K^{1/4}\sqrt{dT}.
\]
If $K<2^{45}$, Lemma~\ref{lem:baseline} gives
$\MinReg\ge c_0\sqrt{dT}$ for a universal constant $c_0>0$.
Since $K^{1/4}\le2^{45/4}$, this also gives
$\MinReg\ge cK^{1/4}\sqrt{dT}$ for a universal constant $c>0$.

\textbf{Large menus.}
Suppose first that $K\ge d\ge2^{45}$ and choose
\begin{equation}\label{eq:product-choice}
m=
\left\lfloor
\min\left\{
\frac{\log K}{\log d},
\frac{d}{2^{45}}
\right\}
\right\rfloor,
\qquad
k=
\left\lfloor
\frac{d}{16m}
\right\rfloor.
\end{equation}
Then $m\ge1$, $k\ge2^{40}$, $16mk\le d$, and
$k^m\le d^m\le K$.
Also $k\ge d/(32m)$, and $m$ is at least half the minimum in
\eqref{eq:product-choice}.
Lemma~\ref{lem:product} yields
\begin{equation}\label{eq:large-lower-combination}
\MinReg
\ge
2^{-16}\sqrt{dT}
\left(\frac{dm}{32}\right)^{1/4}
\ge
2^{-30}\sqrt{dT}
\min\left\{
\sqrt d,
\left(
\frac{d\log K}{\log(2d)}
\right)^{1/4}
\right\}.
\end{equation}

If $2\le d<2^{45}$, Lemma~\ref{lem:baseline} gives
$\MinReg\ge c_0\sqrt{dT}$.
Since the minimum in \eqref{eq:large-lower-combination} is at most
$\sqrt d\le2^{45/2}$, the same lower bound holds with a smaller universal
constant $c>0$.

Decreasing $c$ if necessary proves both claims.
\end{proof}

\begin{remark}[About the auxiliary disclosures]\label{rem:disclosures}
The auxiliary disclosures are used only in the analysis and can be removed
from the actual interaction. Given any learner for the original experiment, define a learner for the auxiliary experiment that simply ignores the extra disclosures.
 This is a valid learner in the auxiliary experiment, so the
lower bound still applies to it.

For each fixed $\theta$, the adversary in the original model can generate
the same auxiliary variables and keep them in its private state. Using these
variables together with the observed history, it generates the same menus as
in the auxiliary construction. This mechanism uses neither future reward
noise nor the learner's current random choice, and is therefore
nonanticipating. Hence the learner sees the same distribution of menus,
actions, and rewards in the two experiments and has the same expected
regret.

Finally, after conditioning the prior on $\{\|\theta\|_2\le1\}$, its support
lies in $B_2^d$. Its average regret is therefore no larger than the supremum
over fixed $\theta\in B_2^d$ and valid menu mechanisms in
\eqref{eq:minimax}. Thus the lower bound also holds in the original model.
\end{remark}

\section{Proofs of the Lower Bound Lemmas}\label{app:lower}
We repeatedly use the following standard probability and information-theoretic
facts. If $Z\sim\N(0,I_r)$, then
$ {\EE e^{\lambda\|Z\|_2^2}}=(1-2\lambda)^{-r/2}$ for $\lambda<1/2$, and hence
\begin{equation}\label{eq:chisq-tails}
 {\PP\{\|Z\|_2^2\notin[r/2,2r]\}}\le2e^{-r/16},
\qquad
 {\PP\{\|Z\|_2^2\ge4r\}}\le e^{-(3/2-\log2)r}.
\end{equation}
For a centered Gaussian random variable of variance at most $v$,
$ {\PP\{|Z|>u\}}\le2e^{-u^2/(2v)}$.

For two probability measures $P$ and $Q$, write
$\TV(P,Q):=\sup_A|P(A)-Q(A)|$ for their total variation distance, and
$ {\KL(P\|Q)}:=\EE_P[\log(dP/dQ)]$ for their Kullback--Leibler divergence
 {when $P$ is absolutely continuous with respect to $Q$, and set
$\KL(P\|Q)=+\infty$ otherwise.}
Pinsker's inequality gives
$\TV(P,Q)\le\sqrt{ {\KL(P\|Q)}/2}$.
We also use that, for any random variable $W\in[0,n]$,
$|\EE_PW-\EE_QW|\le n \TV(P,Q)$.

\subsection{Proof of Lemma~\ref{lem:blocks}}
\label{app:lower-lemma4}
 We first consider the case $m=1$. Fix a block and the history $\cH$ satisfying
\eqref{eq:available-uncertainty}. Set
$D=\lfloor d/2\rfloor-1$, $\rho^2=\sqrt k/D$, and
$\Delta=\sigma k^{1/4}/4$. Then $D\ge4k$, $\rho^2\le1/2$, and
$\Delta^2=\sigma^2\rho^2D/16$.
  The trace condition  in \eqref{eq:available-uncertainty}  implies that at least $\lceil d/2\rceil$ eigenvalues of
$\Sigma$ are at least $\sigma^2/4$. Let $\mathcal E$ be the span of the
corresponding eigenvectors. Since intersecting with $h^\perp$ loses at most
one dimension, we may choose a $D$-dimensional subspace
$\mathcal U\subseteq \mathcal E\cap h^\perp$, where
$D=\lfloor d/2\rfloor-1$. Let $U\in\RR^{d\times D}$ contain an orthonormal
basis of $\mathcal U$. Then
 {$U^\top h=0, \frac{\sigma^2}{4}I_D \preceq U^\top\Sigma U \preceq \sigma^2 I_D$.}
To make the choice of $U$ well defined and measurable, we fix the following
deterministic rule. Order the eigenvalues of $\Sigma$, and within each
eigenspace apply Gram--Schmidt to the projections of the coordinate vectors
$e_1,\ldots,e_d$, skipping zero projections and fixing the sign of each
resulting vector by its first nonzero coordinate. After intersecting the
selected high-variance subspace with $h^\perp$, apply the same rule again to
obtain an orthonormal basis $U$. This makes $U$ a Borel-measurable function of
$(h,\Sigma)$, and hence of the block-start history only; in particular, the
construction uses no future information.
 {Set $S:=U^\top\Sigma U$.}
Draw $J$ uniformly from $[k]$ and independent
$G_1,\ldots,G_k\sim\N(0,I_D)$. Define
\begin{equation}\label{eq:menu-channel}
q_i=
\begin{cases}
\dfrac1{4\sqrt D}
\left(
\rho S^{-1/2}U^\top(\theta-h)
+\sqrt{1-\rho^2}\,G_i
\right),&i=J,\\[2mm]
\dfrac1{4\sqrt D}G_i,&i\ne J.
\end{cases}
\end{equation}
Let $Q=(q_1,\ldots,q_k)$ and define the actions
\begin{equation}\label{eq:normalized-actions}
    x_i:=\frac{Uq_i}{\max\{1,\|q_i\|_2\}}.
\end{equation}
The same menu is used in every round of the block.

To keep the posterior Gaussian from one block to the next, we analyze a
stronger auxiliary experiment in which the learner receives additional
information.  {The unnormalized menu vectors $Q=(q_1,\ldots,q_k)$, but not the seeds
$G_1,\ldots,G_k$, are revealed at the beginning of the block}, while the hidden index $J$ is revealed
only after the block ends, so it cannot help the learner identify the good
action within the block. These disclosures preserve the Gaussian form of the
posterior at block boundaries and only make the learner stronger.
 {Remark~\ref{rem:disclosures} explains how to remove them.}

Since $S^{-1/2}U^\top(\theta-h)\sim\N(0,I_D)$, conditional on any fixed
$J=j$, the hidden column $q_j$ has the same distribution
$\N(0,I_D/(16D))$ as every non-hidden column. Moreover, the columns are
independent. Hence, for every $j\in[k]$,
\begin{equation}\label{eq:Q-marginal}
\mathcal L(Q\mid\cH,J=j)
=
\bigotimes_{i=1}^k \N(0,I_D/(16D)),
\qquad
\PP(J=j\mid\cH,Q)=\frac1k.
\end{equation}
In particular, observing $Q$ does not reveal which index is hidden.

To compute the conditional distribution of the parameter, note that,
conditional on $J=j$,
$4\sqrt D\,q_j=\rho S^{-1/2}U^\top(\theta-h)
+\sqrt{1-\rho^2}\,G_j$.
Thus $S^{-1/2}U^\top(\theta-h)$ and $4\sqrt D\,q_j$ are jointly Gaussian,
with
${\rm{Cov}}(S^{-1/2}U^\top(\theta-h))=I_D$,
${\rm{Cov}}(4\sqrt D\,q_j)=I_D$, and
${\rm{Cov}}(S^{-1/2}U^\top(\theta-h),4\sqrt D\,q_j)=\rho I_D$.
The Gaussian conditioning formula therefore gives
$S^{-1/2}U^\top(\theta-h)\mid Q,J=j
\sim \N(4\rho\sqrt D\,q_j,(1-\rho^2)I_D)$.
Conditioning on the remaining columns of $Q$ does not change this
distribution because, given $J=j$, they are independent of $(\theta,q_j)$.
Finally, since $U^\top h=0$, multiplying by $S^{1/2}$ yields
\begin{equation}\label{eq:conditional-projection}
U^\top\theta \mid Q,J=j
\sim
\N(\mu_j,(1-\rho^2)S),
\qquad \text{where}\
\mu_j=4\rho\sqrt D\,S^{1/2}q_j.
\end{equation}

Conditional on the realized within-block history, the learner's action-selection probabilities
are fixed functions of that history and therefore contribute no additional term to the posterior
likelihood. Thus each observed reward contributes $x_sx_s^\top$ to the posterior precision. Together with
the information revealed by the menu, this gives
\begin{equation}\label{eq:posterior-next-full}
\Sigma_+^{-1}
=
\Sigma^{-1}
+\frac{\rho^2}{1-\rho^2}US^{-1}U^\top
+\sum_{s=1}^n x_sx_s^\top.
\end{equation}
 {Taking traces in \eqref{eq:posterior-next-full} and using
$S\succeq(\sigma^2/4)I_D$ and $\rho^2D=\sqrt k$ gives}
\begin{equation}\label{eq:precision-cost}
\tr(\Sigma_+^{-1})-\tr(\Sigma^{-1})
\le
\frac{8\sqrt k}{\sigma^2}
+\sum_{s=1}^n\|x_s\|_2^2.
\end{equation}
We next show that the random menu is regular, in the sense that it satisfies
several useful geometric properties that simplify the later gap and
information calculations, with constant probability. Define
\begin{equation}\label{eq:good-menu-full}
\mathsf G
:=
\left\{
\frac1{32}\le \|q_i\|_2^2\le \frac18
\ \forall i\in[k],\;
\|Q\|_{\op}\le\frac34,\;
|q_i^\top S^{1/2}q_j|\le\frac{\sigma}{256}
\ \forall i\ne j
\right\}.
\end{equation}
We now verify that $\mathsf G$ occurs with constant probability. By
\eqref{eq:Q-marginal}, conditional on $\cH$, the columns
$q_1,\ldots,q_k$ are independent and each has distribution
$\N(0,I_D/(16D))$. We use this common Gaussian law to verify the three
conditions defining $\mathsf G$.
For every $i$, we have
$16D\|q_i\|_2^2\sim\chi_D^2$. Therefore, by
\eqref{eq:chisq-tails},
 {$\PP\left( \|q_i\|_2^2\notin[1/32,1/8] \mid\cH \right) \le 2e^{-D/16}$.}
A union bound over $i\in[k]$ gives a total contribution at most
$2ke^{-D/16}$.

Next consider $\|Q\|_{\op}$. Let $\mathcal N$ be a $1/4$-net of the unit
sphere in $\RR^k$ with $|\mathcal N|\le 9^k$. The standard net bound gives
$\|Q\|_{\op}\le (4/3)\max_{v\in\mathcal N}\|Qv\|_2$. Hence
$\|Q\|_{\op}>3/4$ implies $\|Qv\|_2>9/16$ for some
$v\in\mathcal N$. For every fixed unit vector $v$,
$Qv\sim\N(0,I_D/(16D))$, so
$4\sqrt D\,Qv\sim\N(0,I_D)$. Therefore
 {$\PP\left( \|Qv\|_2>\frac9{16} \mid\cH \right) \le e^{-(3/2-\log2)D}$,}
where we used the second inequality in \eqref{eq:chisq-tails}. A union bound
over $\mathcal N$ gives
 {$\PP\left( \|Q\|_{\op}>\frac34 \mid\cH \right) \le 9^k e^{-(3/2-\log2)D} \le e^{-D/10}$,}
where the last inequality follows from $D\ge4k$.

It remains to control the cross terms
$q_i^\top S^{1/2}q_j$ for $i\ne j$.  Fix $i\ne j$.
The vectors $q_i$ and $q_j$ are independent. Conditional on $q_i$,
$q_i^\top S^{1/2}q_j$ is therefore a centered Gaussian with variance
$q_i^\top S q_i/(16D)$. On the event
$\|q_i\|_2^2\le1/8$, since $S\preceq\sigma^2I_D$, this variance is at most
$\sigma^2/(128D)$. The scalar Gaussian tail bound therefore gives
 {$\PP\left( |q_i^\top S^{1/2}q_j|>\frac{\sigma}{256}, \ \|q_i\|_2^2\le\frac18 \mid\cH \right) \le 2e^{-D/1024}$.}
Taking a union bound over all ordered pairs $i\ne j$ gives a contribution at
most $2k^2e^{-D/1024}$.

Combining the three bounds,
\begin{equation}\label{eq:good-menu-prob}
\PP(\mathsf G^c\mid\cH)
\le
2ke^{-D/16}
+
e^{-D/10}
+
2k^2e^{-D/1024}
\le
\frac14.
\end{equation}
The last inequality uses $D\ge4k$ and $k\ge2^{40}$. On $\mathsf G$, $\|q_i\|_2<1$ for every $i$. Hence the normalization in
$x_i=Uq_i/\max\{1,\|q_i\|_2\}$ is inactive, so
$x_i=Uq_i$ and therefore $U^\top x_i=q_i$.

\paragraph{One reward reveals little about the hidden index.}
Fix $Q\in\mathsf G$. For each $j\in[k]$, let $P_j$ denote the law of the
within-block interaction conditional on $J=j$.  {By \eqref{eq:conditional-projection}, under $P_j$ we have}
$U^\top\theta=\mu_j+\zeta$, where
$\zeta\sim\N(0,(1-\rho^2)S)$ is drawn once and remains fixed throughout the
block. To measure how much the rewards reveal about the hidden index, define
a reference law $P_0$ using the same menu, the same learner policy, and the
same residual $\zeta$, but removing the planted shift $\mu_j$; that is, under
$P_0$ we set $U^\top\theta=\zeta$. Thus $P_j$ and $P_0$ differ only through
the shift $\mu_j$.
 {Divergences below concern the marginal interaction laws. Probabilities of
residual events refer to the same experiments with their latent residuals retained.}

 Conditional on $\zeta$ and the past, when
the learner selects action $x_{A_s}$, the reward distributions under $P_j$
and $P_0$ differ only in their means, by
$(U^\top x_{A_s})^\top\mu_j$. Applying the Gaussian KL, the KL chain
rule, and then data processing after removing $\zeta$ gives
\begin{equation}\label{eq:KL-latent}
\KL(P_0\|P_j)
\le
\frac12\EE_0\sum_{s=1}^n
\left((U^\top x_{A_s})^\top\mu_j\right)^2.
\end{equation}
The key geometric fact is that one selected action cannot be strongly aligned
with all $k$ possible hidden shifts. On $\mathsf G$, we have
$U^\top x_i=q_i$ and
$\mu_j=4\rho\sqrt D\,S^{1/2}q_j$. Therefore

\begin{equation}\label{eq:shift-energy}
\begin{aligned}
\sum_{j=1}^k\bigl((U^\top x_i)^\top\mu_j\bigr)^2
&=16\rho^2D\|Q^\top S^{1/2}q_i\|_2^2\\
&\le16\rho^2D\|Q\|_{\op}^2\|S^{1/2}\|_{\op}^2\|q_i\|_2^2\\
&\le\frac98\sigma^2\rho^2D\\
&=18\Delta^2.
\end{aligned}
\end{equation}
 where the last equality uses
$\Delta^2=\sigma^2\rho^2D/16$. Averaging \eqref{eq:KL-latent} over the uniform hidden index gives
\begin{equation}\label{eq:KL-average-full}
\frac1k\sum_{j=1}^k\KL(P_0\|P_j)
\le
\frac{9n\Delta^2}{k}.
\end{equation}
Since $\Delta^2=\sigma^2\sqrt k/16$, the assumption
$n\sigma^2/\sqrt k\le1/64$ implies
$n\Delta^2/k\le1/1024$.

Let $N_j$ be the number of pulls of action $j$ in the block.
Since $0\le N_j\le n$, the expectation--TV inequality gives
$\EE_j[N_j]\le\EE_0[N_j]+n\,\TV(P_0,P_j)$. Averaging over $j$, using
 {$\sum_jN_j\le n$, Pinsker's inequality, Cauchy--Schwarz, and
\eqref{eq:KL-average-full},}
\begin{align}
\frac1k\sum_{j=1}^k\EE_j [N_j]
&\le
\frac nk
+
\frac nk\sum_{j=1}^k\TV(P_0,P_j)\notag\\
&\le
\frac nk
+
n\sqrt{
\frac1{2k}
\sum_{j=1}^k\KL(P_0\|P_j)
}\notag\\
&\le
\frac{5n}{8}.
\label{eq:Nj-bound}
\end{align}

\paragraph{The hidden action has a definite gap.}
 {We show that the hidden action has a gap of at least $\Delta/8$.} Recall that $\mu_j=4\rho\sqrt D\,S^{1/2}q_j$. Since
$S\succeq(\sigma^2/4)I_D$, we have
$S^{1/2}\succeq(\sigma/2)I_D$. Hence, on $\mathsf G$,
 {$q_j^\top\mu_j = 4\rho\sqrt D\,q_j^\top S^{1/2}q_j \ge 4\rho\sqrt D\frac{\sigma}{2}\|q_j\|_2^2 \ge \frac{\sigma\rho\sqrt D}{16} = \frac{\Delta}{4}$.}
For $i\ne j$, the  condition in
\eqref{eq:good-menu-full}  gives
 {$|q_i^\top\mu_j| = 4\rho\sqrt D\,|q_i^\top S^{1/2}q_j| \le 4\rho\sqrt D\frac{\sigma}{256} = \frac{\Delta}{16}$.}
Therefore
\begin{equation}\label{eq:own-signal}
q_j^\top\mu_j\ge\frac{\Delta}{4},
\qquad
|q_i^\top\mu_j|\le\frac{\Delta}{16}
\quad(i\ne j).
\end{equation}
Now consider the residual
$\zeta\sim\N(0,(1-\rho^2)S)$. For every $i$,
 {$\operatorname{Var}(q_i^\top\zeta) = (1-\rho^2)q_i^\top S q_i \le \sigma^2\|q_i\|_2^2 \le \frac{\sigma^2}{8}$.}
Thus, by the  Gaussian tail bound and a union bound,

\begin{equation}\label{eq:residual-tail}
P_j\left(\max_{i\le k}|q_i^\top\zeta|>\frac\Delta{32}\right)
\le2k\exp\left(-\frac{\Delta^2}{256\sigma^2}\right)
=2k\exp\left(-\frac{\sqrt k}{4096}\right)\le\frac1{16}.
\end{equation}

Hence, with probability at least $15/16$,
$|q_i^\top\zeta|\le\Delta/32$ simultaneously for all $i$. On this event,
 {$x_j^\top\theta = q_j^\top(\mu_j+\zeta) \ge \frac{\Delta}{4}-\frac{\Delta}{32} = \frac{7\Delta}{32}$,}
while for every $i\ne j$,
 {$x_i^\top\theta = q_i^\top(\mu_j+\zeta) \le \frac{\Delta}{16}+\frac{\Delta}{32} = \frac{3\Delta}{32}$.}
Therefore action $j$ is optimal and its gap to every other principal action
is at least $\Delta/8$.

Let
$\mathsf E_j=\{\max_{i\le k}|q_i^\top\zeta|\le\Delta/32\}$.
Since $0\le n-N_j\le n$, we do not need any independence between $N_j$ and
$\mathsf E_j$:

\begin{equation}\label{eq:event-subtraction}
\begin{aligned}
\EE_j[R_{\mathrm{block}}]
&\ge
\frac{\Delta}{8}
\EE_j\left[(n-N_j)\1_{\mathsf E_j}\right]\\
&\ge
\frac{\Delta}{8}
\left(
n-\EE_jN_j-nP_j(\mathsf E_j^c)
\right).
\end{aligned}
\end{equation}

 {Averaging \eqref{eq:event-subtraction} over $j$ and using
\eqref{eq:Nj-bound} and $P_j(\mathsf E_j^c)\le1/16$ gives}
\begin{equation}\label{eq:block-good-regret}
\frac1k\sum_{j=1}^k\EE_j[R_{\mathrm{block}}]
\ge
\frac{\Delta}{8}
\left(
n-\frac{5n}{8}-\frac{n}{16}
\right)
=
\frac{5n\Delta}{128}.
\end{equation}

 {The bound \eqref{eq:block-good-regret} holds for every $Q\in\mathsf G$.} Since
$\PP(J=j\mid\cH,Q)=1/k$, averaging over the hidden index is exactly the
conditional expected regret given $\cH$ and $Q$. Finally,
$\PP(\mathsf G\mid\cH)\ge3/4$, and regret is nonnegative, so
\begin{equation}\label{eq:one-block-full}
\EE[R_{\mathrm{block}}\mid\cH]
\ge
\frac34\cdot\frac{5n\Delta}{128}
\ge
\frac{n\Delta}{64}
=
\frac{n\sigma k^{1/4}}{256}.
\end{equation}

 {Equation~\eqref{eq:one-block-full} establishes the one-group regret bound for the $k$ principal actions.} The padding
needed when $K>k$ is handled in the general construction below, which also applies when $m=1$.

\textbf{General $m$.}
We now combine $m$ independent hidden choices. Set
$D=\lfloor d/(2m)\rfloor-1$, $\rho^2=\sqrt k/D$, and
$\Delta=\sigma k^{1/4}/4$. The condition $16mk\le d$ implies $D\ge4k$. Choose covariance-orthogonal frames
$U_1,\ldots,U_m\in\RR^{d\times D}$ with
\begin{equation}\label{eq:product-orthogonality}
U_\ell^\top U_r=0,
\qquad
U_\ell^\top\Sigma U_r=0\quad(\ell\ne r),
\qquad
\frac{\sigma^2}{4}I_D
\preceq
U_\ell^\top\Sigma U_\ell
\preceq
\sigma^2I_D,
\end{equation}
and $U_\ell^\top h=0$.

Such frames can be constructed as follows. By \eqref{eq:available-uncertainty}, the high-variance eigenspace of $\Sigma$ has
dimension at least $\lceil d/2\rceil$. Using the same measurable eigenbasis selection as above, choose $m$
mutually orthogonal subspaces $\mathcal V_1,\ldots,\mathcal V_m$, each spanned by $D+1=\lfloor d/(2m)\rfloor$ of these eigenvectors.
This is possible because $m(D+1)\le d/2$. For each $\ell$, the intersection $\mathcal V_\ell\cap h^\perp$ has dimension at least
$D$; choose $U_\ell$ as an orthonormal basis of a $D$-dimensional subspace of this intersection. Since the $\mathcal V_\ell$
are spanned by disjoint sets of eigenvectors, the resulting frames satisfy $U_\ell^\top U_r=0$ and $U_\ell^\top\Sigma U_r=0$
for $\ell\ne r$, as well as the remaining properties in \eqref{eq:product-orthogonality}.

In group $\ell$, draw an independent hidden index $J_\ell\in[k]$ and generate
$q_{\ell,1},\ldots,q_{\ell,k}$ by the one-group channel
\eqref{eq:menu-channel}, with $U,S,J$ replaced by
$U_\ell,U_\ell^\top\Sigma U_\ell,J_\ell$. For
$\boldsymbol i=(i_1,\ldots,i_m)\in[k]^m$, define
\begin{equation}\label{eq:product-actions}
x_{\boldsymbol i}
=
\frac1{\sqrt m}
\sum_{\ell=1}^m
U_\ell\bar q_{\ell,i_\ell},
\qquad
\bar q_{\ell,i}
=
\frac{q_{\ell,i}}{\max\{1,\|q_{\ell,i}\|_2\}}.
\end{equation}

The orthogonality of the frames and $\|\bar q_{\ell,i}\|_2\le1$ in
\eqref{eq:product-actions} give
$\|x_{\boldsymbol i}\|_2^2=m^{-1}\sum_\ell\|\bar q_{\ell,i_\ell}\|_2^2\le1$.
These $k^m$ principal actions are distinct almost surely. If $K>k^m$, add
distinct points on segments between principal actions. One deterministic
rule is to enumerate rational coefficients on the segment joining the first
two principal actions and skip points already in the menu.
For a principal action, set
$w_{\ell,j}(x_{\boldsymbol i}):=\1_{\{i_\ell=j\}}$.
For a padded action $a=\lambda x_{\boldsymbol i}+(1-\lambda)x_{\boldsymbol i'}$,
use the same interpolation for $w_{\ell,j}(a)$.
Then $0\le w_{\ell,j}(a)\le1$ and $\sum_jw_{\ell,j}(a)=1$.
Write $v_\ell(a):=\sum_jw_{\ell,j}(a)\bar q_{\ell,j}$, so that
$a=m^{-1/2}\sum_\ell U_\ell v_\ell(a)$ for every action in the menu.

Let $Q_\ell=(q_{\ell,1},\ldots,q_{\ell,k})$ and
$S_\ell:=U_\ell^\top\Sigma U_\ell$.
The standardized residuals in the $m$ subspaces are jointly Gaussian and
uncorrelated, hence independent. The one-group calculation therefore shows
that the hidden labels remain independent and uniform after all the raw
menus $Q_1,\ldots,Q_m$ are observed. Conditional on those menus and labels,
the projections have the form
$U_\ell^\top\theta=\mu_{\ell,J_\ell}+\zeta_\ell$, where
$\mu_{\ell,j}:=4\rho\sqrt D\,S_\ell^{1/2}q_{\ell,j}$ and the residuals
$\zeta_\ell\sim\N(0,(1-\rho^2)S_\ell)$ are independent across groups.
The full posterior is Gaussian, and after the rewards its precision is
\begin{equation}\label{eq:product-next-precision}
\Sigma_+^{-1}=\Sigma^{-1}
+\frac{\rho^2}{1-\rho^2}\sum_{\ell=1}^m U_\ell S_\ell^{-1}U_\ell^\top
+\sum_{s=1}^n x_{A_s}x_{A_s}^\top.
\end{equation}
Taking traces in \eqref{eq:product-next-precision}, as in
\eqref{eq:precision-cost}, gives
\begin{equation}\label{eq:product-menu-trace}
\tr(\Sigma_+^{-1})\le\tr(\Sigma^{-1})
+\frac{8m\sqrt k}{\sigma^2}+\sum_{s=1}^n\|x_{A_s}\|_2^2,
\end{equation}
which proves \eqref{eq:reusable-precision}.

\paragraph{Testing one group.}
Fix $\ell$. Let $\mathsf G_\ell$ be the event in
\eqref{eq:good-menu-full} with $(Q,S)$ replaced by $(Q_\ell,S_\ell)$.
Then $\PP(\mathsf G_\ell\mid\cH)\ge3/4$ by
\eqref{eq:good-menu-prob}. Condition on all raw menus with
$Q_\ell\in\mathsf G_\ell$ and on the other labels $(J_r)_{r\ne\ell}$.\footnote{Only group $\ell$ is restricted to a regular menu; the other groups are averaged without such a restriction.}
Under this conditioning, $J_\ell$ is still uniform.
Let $P_j$ be the law of the within-block interaction when $J_\ell=j$.
Let $P_0$ use the same policy, the same residual vector
$(\zeta_1,\ldots,\zeta_m)$, and the same mean shifts in all other groups,
but replace $U_\ell^\top\theta$ by $\zeta_\ell$.
These are conditional laws used for analysis; the learner is not given
$J_\ell$ during the block.

For action $a$, the difference between the two reward means is
$m^{-1/2}v_\ell(a)^\top\mu_{\ell,j}$.
On $\mathsf G_\ell$, normalization in group $\ell$ is inactive.
The bound \eqref{eq:shift-energy}, followed by convexity for padded actions,
gives $\sum_j(v_\ell(a)^\top\mu_{\ell,j})^2\le18\Delta^2$ for every
available action $a$. Applying the Gaussian KL chain rule conditional on
the common residual vector, and then removing that vector by data processing,
yields
\begin{equation}\label{eq:product-KL-average}
\frac1k\sum_{j=1}^k\KL(P_0\|P_j)
\le\frac1{2mk}\EE_0\sum_{s=1}^n\sum_{j=1}^k
\bigl(v_\ell(x_{A_s})^\top\mu_{\ell,j}\bigr)^2
\le\frac{9n\Delta^2}{mk}.
\end{equation}
Since $n\sigma^2/(m\sqrt k)\le1/64$ and
$\Delta^2=\sigma^2\sqrt k/16$, we have $n\Delta^2/(mk)\le1/1024$.
The random counts $\sum_s w_{\ell,j}(x_{A_s})$ lie in $[0,n]$ and sum
to $n$ over $j$. The expectation--TV inequality, together with
\eqref{eq:product-KL-average}, gives
\begin{equation}\label{eq:product-fractional-count}
\frac1k\sum_{j=1}^k\EE_j\sum_{s=1}^n w_{\ell,j}(x_{A_s})
\le\frac nk+n\sqrt{\frac1{2k}\sum_{j=1}^k\KL(P_0\|P_j)}
\le\frac{5n}{8}.
\end{equation}

\paragraph{Adding the group regrets.}
For any action $a$, define its unscaled regret in group $\ell$ by
$D_\ell(a):=\max_{j\in[k]}\bar q_{\ell,j}^\top U_\ell^\top\theta
-v_\ell(a)^\top U_\ell^\top\theta$.
It is nonnegative, including for padded actions. Maximization over the
product menu separates across groups, and padding cannot increase the
maximum of a linear function. Consequently,
\begin{equation}\label{eq:group-regret-decomposition}
R_{\mathrm{block}}=\frac1{\sqrt m}
\sum_{\ell=1}^m\sum_{s=1}^n D_\ell(x_{A_s}).
\end{equation}
Under the conditioning above, let
$\mathsf E_{\ell,j}:=\{\max_i|q_{\ell,i}^\top\zeta_\ell|\le\Delta/32\}$.
 {Applying \eqref{eq:residual-tail} in group $\ell$ gives
$P_j(\mathsf E_{\ell,j}^c)\le1/16$.}
 {On $\mathsf E_{\ell,j}$, the within-group gap calculation in \eqref{eq:own-signal}}
gives $D_\ell(a)\ge\Delta(1-w_{\ell,j}(a))/8$ when $J_\ell=j$.
Since these regrets are nonnegative outside the event, we may subtract the
failure probability without assuming independence from the chosen actions.
Using \eqref{eq:product-fractional-count} gives
\begin{equation}\label{eq:group-regret-bound}
\begin{aligned}
\frac1k\sum_{j=1}^k\EE_j\sum_{s=1}^n D_\ell(x_{A_s})
&\ge\frac\Delta8\left(n-\frac1k\sum_{j=1}^k
\EE_j\sum_{s=1}^n w_{\ell,j}(x_{A_s})-\frac n{16}\right)\\
&\ge\frac{5n\Delta}{128}.
\end{aligned}
\end{equation}
The bound \eqref{eq:group-regret-bound} holds for every choice of the other
raw menus and labels, provided
$Q_\ell\in\mathsf G_\ell$. Averaging over them and using
$\PP(\mathsf G_\ell\mid\cH)\ge3/4$ gives
$\EE[\sum_sD_\ell(x_{A_s})\mid\cH]\ge n\Delta/64$.
Linearity of expectation in \eqref{eq:group-regret-decomposition} now gives
\begin{equation}\label{eq:product-block-full}
\EE[R_{\mathrm{block}}\mid\cH]
\ge\frac1{\sqrt m}\sum_{\ell=1}^m\frac{n\Delta}{64}
=\frac{n\sigma\sqrt m\,k^{1/4}}{256}.
\end{equation}
Equations~\eqref{eq:product-menu-trace} and \eqref{eq:product-block-full}
prove Lemma~\ref{lem:blocks}.

\subsection{Proof of Lemma~\ref{lem:truncation}}\label{app:truncation}
\begin{proof}[Proof of Lemma~\ref{lem:truncation}]
Since all actions lie in the unit ball, $0\le R_T\le2T\|\theta\|_2$. Write $\theta=\sigma Z$ with $Z\sim\N(0,I_d)$.
By Cauchy--Schwarz and the Gaussian tail bound,
 {$\EE[\|Z\|_2\1_{\{\|Z\|_2>1/\sigma\}}]\le\sqrt d\,2^{d/4}e^{-1/(8\sigma^2)}<2^{-11}$,}
where the last inequality uses $T\ge d^2$ and $\sigma^2=d/(4096T)$. Therefore
 {$\EE[R_T\1_{\{\|\theta\|_2>1\}}]\le2T\sigma\,2^{-11}=T\sigma/1024$.}
\end{proof}

\subsection{Proof of Lemma~\ref{lem:product}}\label{app:product-proof}
\begin{proof}[Proof of Lemma~\ref{lem:product}]
Set $\sigma^2=d/(4096T)$, $B=\lfloor d/(32m\sqrt k)\rfloor$, and $n=\lfloor T/B\rfloor$, and draw $\theta\sim
\N(0,\sigma^2I_d)$ at the start. This gives $d/(64m\sqrt k)\le B\le d/(32m\sqrt k)$, $Bn\ge T/2$, and $n\le64Tm\sqrt k/d$.
 {Hence $n\sigma^2/(m\sqrt k)\le1/64$, so every block has the length required by Lemma~\ref{lem:blocks}.
In the last $T-Bn$ rounds, repeat the final menu. These rounds add nonnegative regret and no extra menu changes.} It remains to check
that the posterior uncertainty is never exhausted. Initially $\tr(\Sigma^{-1})=d/\sigma^2$. After at most $B$ blocks,
\eqref{eq:reusable-precision} and $\|x_t\|_2\le1$ give
\begin{align}
\tr(\Sigma^{-1})&\le\frac d{\sigma^2}+\frac{8Bm\sqrt k}{\sigma^2}+T\notag\\
&\le\left(1+\frac14+\frac1{4096}\right)\frac d{\sigma^2}<\frac{2d}{\sigma^2}.\label{eq:trace-resource}
\end{align}
 {By \eqref{eq:trace-resource}, the hypothesis of Lemma~\ref{lem:blocks} holds at every block boundary, so the construction can be repeated}
using the same fixed parameter. Summing \eqref{eq:reusable-regret} over the blocks and using $Bn\ge T/2$,
\begin{equation}\label{eq:product-untruncated}
\EE[R_T]\ge\frac{Bn\sigma\sqrt m\,k^{1/4}}{256}\ge\frac{T\sigma\sqrt m\,k^{1/4}}{512}.
\end{equation}
 {Subtract the tail bound in Lemma~\ref{lem:truncation} from
\eqref{eq:product-untruncated}.
The remaining expected regret is nonnegative. Dividing it by
$\PP(\|\theta\|_2\le1)\le1$ can only increase it, so}
\begin{equation}\label{eq:product-truncated}
\EE[R_T\mid\|\theta\|_2\le1]\ge\frac{T\sigma\sqrt m\,k^{1/4}}{512}-\frac{T\sigma}{1024}
\ge2^{-16}\sqrt m\,k^{1/4}\sqrt{dT}.
\end{equation}

Remark~\ref{rem:disclosures} removes the auxiliary disclosures. The conditioned
prior in \eqref{eq:product-truncated} is supported in the unit ball, so its expected regret is at most the
supremum over fixed parameters and menu mechanisms in \eqref{eq:minimax}.
This proves \eqref{eq:product-lower}.

\end{proof}

\section{Proof of the Upper Bounds}\label{sec:upper-analysis}
The proof uses three ingredients: the repair family contains one
good repair sequence, the master can compete with this repair expert after the trajectory is known,
and the remaining uncertainty can be charged only to the actions actually played. The technical
proofs are deferred to Appendix~\ref{app:upper}, except for the standard OFUL
guarantee, which is invoked directly from \citet{abbasi2011}.

\subsection{Key Technical Lemmas for Small Menus}
We first show that the repair family $\cE_M$ contains a good repair sequence that remains accurate on
every displayed action throughout the horizon. The following lemma formalizes this property.
\begin{lemma}\label{lem:budget}
With probability at least $1-T^{-2}$, there exists $e^\star\in\cE_M$ such that
\begin{equation}\label{eq:good-repair}
|x_{t,i}^\top(\theta-v_{t,e^\star})|\le\alpha w_{t,i}\quad\text{for every }t\in[T]\text{ and }i\in[K].
\end{equation}
\end{lemma}
The proof of Lemma~\ref{lem:budget} is given in Appendix~\ref{app:budget}. For the small-menu algorithm, recall $\ell_t(i):=
\Phi(-\mu_{t,i})$ and $\bar\ell_t(i):=(\ell_t(i)+1-\beta_{t,i})/2$.  {We use the following bound from the implicit-exploration analysis of
\citet{neu2015}.}
\begin{lemma}\label{lem:ix}
Algorithm~\ref{alg:upper} satisfies
\begin{equation}\label{eq:ix-master-bound}
\EE\left[\sum_{t=1}^T\bar\ell_t(A_t)-\min_{e\in\cE_M}\sum_{t=1}^T\bar\ell_t(a_{t,e})\right]
\le\frac{2\log|\cE_M|}{\eta}+\frac32\eta KT.
\end{equation}
\end{lemma}
Importantly, the minimum is inside the expectation. Thus the comparator may depend on the
realized trajectory, which allows us to use the random good repair sequence $e^\star$ from Lemma~\ref{lem:budget}. Let
$a_t^\star:=a_{t,e^\star}$ denote the action recommended by the good repair expert, and let $i_t^\star$ denote an optimal
action in round $t$.
\begin{lemma}\label{lem:bridge}
Suppose a repair expert $e^\star$ satisfies $|x_{t,i}^\top(\theta-v_{t,e^\star})|\le\alpha w_{t,i}$ for every $t\in[T]$ and $i\in[K]$,
and let $a_t^\star:=a_{t,e^\star}$. Then
\begin{equation}\label{eq:domination}
\bar\ell_t(a_t^\star)\le\frac{\ell_t(i_t^\star)+1}{2}\quad\text{for every }t\in[T].
\end{equation}
For every realized action sequence,
\begin{equation}\label{eq:played-uncertainty}
\sum_{t=1}^T\beta_{t,A_t}\le\alpha\sqrt{2Td\log(1+T/d)}.
\end{equation}
For Repair-Geo, the same repair expert also satisfies
\begin{equation}\label{eq:geo-domination}
g_t(a_t^\star)\ge\frac{\mu_{t,i_t^\star}-1}{2}\quad\text{for every }t\in[T].
\end{equation}
\end{lemma}
Indeed, \eqref{eq:good-repair} and optimism imply $\mu_{t,i_t^\star}-\mu_{t,a_t^\star}\le2\alpha w_{t,a_t^\star}$. The discounted loss absorbs this
uncertainty and gives \eqref{eq:domination}. The bound \eqref{eq:played-uncertainty} follows from the usual elliptical-potential inequality
applied only to the actions actually played.

\subsection{Proof of Theorem~\ref{thm:upper}}
\begin{proof}
For every round, $\ell_t(A_t)-\ell_t(i_t^\star)=2\bigl(\bar\ell_t(A_t)-(\ell_t(i_t^\star)+1)/2\bigr)+\beta_{t,A_t}$. Summing over time and
inserting the best repair expert gives
\begin{align}
\sum_{t=1}^T\bigl(\ell_t(A_t)-\ell_t(i_t^\star)\bigr)
={}&2\left[\sum_{t=1}^T\bar\ell_t(A_t)-\min_{e\in\cE_M}\sum_{t=1}^T\bar\ell_t(a_{t,e})\right]\notag\\
&+2\left[\min_{e\in\cE_M}\sum_{t=1}^T\bar\ell_t(a_{t,e})-\sum_{t=1}^T\frac{\ell_t(i_t^\star)+1}{2}\right]
+\sum_{t=1}^T\beta_{t,A_t}.\label{eq:ix-decomposition}
\end{align}
 {Lemma~\ref{lem:ix} controls the first term in \eqref{eq:ix-decomposition}. On the event in
Lemma~\ref{lem:budget}, the second term is nonpositive by
\eqref{eq:domination}. On the failure event} it is at most $2T$, so its expected contribution is at most $2/T$. Finally, \eqref{eq:played-uncertainty} controls
the last term. Therefore
\begin{equation}\label{eq:ell-regret}
\EE\sum_{t=1}^T\bigl(\ell_t(A_t)-\ell_t(i_t^\star)\bigr)
\le\frac{4\log|\cE_M|}{\eta}+3\eta KT+\alpha\sqrt{2Td\log(1+T/d)}+\frac2T.
\end{equation}
Since $|\cE_M|\le(2KT+1)^M$, the choices of $\alpha$ and $\eta$ in Theorem~\ref{thm:upper} imply
 {$\frac{4\log|\cE_M|}{\eta}+3\eta KT\le7\sqrt{KTM\log(2KT+1)}$.}
Moreover, all means lie in $[-1,1]$, so $\ell_t(A_t)-\ell_t(i_t^\star)\ge\phi(1)(\mu_{t,i_t^\star}-\mu_{t,A_t})$, )  wher $\phi$ denote the density of \(N(0,1)\). Substituting $\alpha=
[K\log(2KT+1)]^{1/4}$ and the definition of $M$ into \eqref{eq:ell-regret} yields
 {$\EE[R_T]\le130\,[K\log(2KT+1)]^{1/4}\sqrt{dT\log(1+T/d)}$.}
The  constant calculation is given in Appendix~\ref{app:constants}.
\end{proof}
\subsection{Key Technical Lemmas for Large Menus}\label{sec:large-lemmas}
For large menus, the repair argument is unchanged. We only need a master whose variance depends
on the rank of the lifted features rather than directly on $K$.
\begin{lemma}\label{lem:geo}
 {Suppose the current lifted features $z_{t,i}$ span a nonzero subspace
of dimension at most $\min\{K,d+1\}$. For some fixed vector $\vartheta$,
let their mean rewards satisfy $g_t(i)=z_{t,i}^\top\vartheta\in[-1,1]$,
and suppose the observation has mean $g_t(i)$ and conditionally
$1$-subGaussian centered noise after action $i$ is chosen.} If $0<\eta\le1/(16\min\{K,d+1\})$, $\gamma=8\min\{K,d+1\}\eta$,
and the master uses \eqref{eq:geo-distribution} and \eqref{eq:geo-update}
with the designs from Lemma~\ref{lem:design}, then
\begin{equation}\label{eq:geo-master-bound}
\EE\left[\max_{e\in\cE_M}\sum_{t=1}^T g_t(a_{t,e})-\sum_{t=1}^T g_t(A_t)\right]
\le\frac{2\log|\cE_M|}{\eta}+32\eta\min\{K,d+1\}T.
\end{equation}
\end{lemma}
With $\eta=\sqrt{M\log(2KT+1)/(16\min\{K,d+1\}T)}$ and $|\cE_M|\le(2KT+1)^M$, the right-hand side of \eqref{eq:geo-master-bound} is
at most $16\sqrt{\min\{K,d+1\}TM\log(2KT+1)}$ whenever $M\log(2KT+1)\le T/(16\min\{K,d+1\})$.
The following lemma is the standard OFUL guarantee of
\citet[Theorems~2 and~3]{abbasi2011}, specialized to unit parameter-norm,
action-norm, and noise bounds. The expected-regret form below includes the
 {failure-event contribution $2T\cdot T^{-2}=2/T$.}
\begin{lemma}[OFUL \citep{abbasi2011}]\label{lem:oful}
For $d,K\ge2$ and $T\ge d$, run OFUL with regularization $\lambda=1$ and
failure probability $\delta=T^{-2}$. Initialize $V_1=I_d$ and $b_1=0$.
At round $t$, set $\widehat\theta_t=V_t^{-1}b_t$ and choose
\begin{equation}\label{eq:oful-rule}
A_t\in\argmax_{i\in[K]}\left\{x_{t,i}^\top\widehat\theta_t+
\left(1+\sqrt{\log\det V_t+4\log T}\right)w_{t,i}\right\}.
\end{equation}
After observing $Y_t$, update $V_{t+1}=V_t+x_tx_t^\top$ and $b_{t+1}=b_t+x_tY_t$.
Under $\|\theta\|_2\le1$, $\|x_{t,i}\|_2\le1$, and conditionally
$1$-subGaussian centered noise, this algorithm satisfies
\begin{equation}\label{eq:oful-bound}
\EE[R_T]\le4\sqrt{Td\log(1+T/d)}
\left(1+\sqrt{d\log(1+T/d)+4\log T}\right)+\frac2T.
\end{equation}
In particular, $\EE[R_T]\le16d\log(1+T/d)\sqrt T$.
\end{lemma}
The proof of Lemma~\ref{lem:geo} is given in Appendix~\ref{app:geometry}.
Lemma~\ref{lem:oful} is quoted from \citet{abbasi2011}; the elementary
simplification of its constant is computed in Appendix~\ref{app:constants}.

\subsection{Proof of Theorem~\ref{thm:large}}\label{sec:proof-large}
\begin{proof}
Recall $\alpha=[\min\{K,d+1\}\log(2KT+1)]^{1/4}$. We distinguish two cases.

\textbf{Case 1: $\log K\le d$ and $M\log(2KT+1)\le T/(16\min\{K,d+1\})$.}
Run Algorithm~\ref{alg:geo} with $\eta=\sqrt{M\log(2KT+1)/(16\min\{K,d+1\}T)}$.
The second inequality implies $\eta\le1/(16\min\{K,d+1\})$.
The lifted features $z_{t,1},\ldots,z_{t,K}$ have rank at most $\min\{K,d+1\}$, and the lifted observation $Z_t$ satisfies the assumptions of Lemma~\ref{lem:geo}.

The reward regret admits the decomposition
\begin{align}
R_T={}&2\left[\max_{e\in\cE_M}\sum_{t=1}^Tg_t(a_{t,e})-\sum_{t=1}^Tg_t(A_t)\right]\notag\\
&+\left[\sum_{t=1}^T(\mu_{t,i_t^\star}-1)-2\max_{e\in\cE_M}\sum_{t=1}^Tg_t(a_{t,e})\right]
+2\sum_{t=1}^T\beta_{t,A_t}.\label{eq:geo-decomposition}
\end{align}
 {For the geometric branch, Lemma~\ref{lem:geo} and $|\cE_M|\le(2KT+1)^M$ bound the expected first term in \eqref{eq:geo-decomposition}}
by $32\sqrt{\min\{K,d+1\}TM\log(2KT+1)}$. On the event of Lemma~\ref{lem:budget}, the second term is nonpositive by \eqref{eq:geo-domination}; its
failure event contributes at most $2/T$. Finally, \eqref{eq:played-uncertainty} bounds the last term. Thus
\begin{equation}\label{eq:geo-general}
\EE[R_T]\le32\sqrt{\min\{K,d+1\}TM\log(2KT+1)}+2\alpha\sqrt{2Td\log(1+T/d)}+\frac2T.
\end{equation}
When $K\ge d$, we have $\min\{K,d+1\}\le3d/2$. Since $\log K\le d$, substituting the
definition of $M$ into \eqref{eq:geo-general} and using the elementary bounds in Appendix~\ref{app:constants} gives
 {$\EE[R_T]\le167\log(2dT)\,d^{3/4}\sqrt T\,(\log K)^{1/4}$.}

\textbf{Case 2: $\log K>d$ or $M\log(2KT+1)>T/(16\min\{K,d+1\})$.}
Run OFUL with $\lambda=1$ and $\delta=T^{-2}$. If $\log K>d$, the standard
OFUL guarantee in Lemma~\ref{lem:oful} gives
 {$\EE[R_T]\le16\log(2dT)\,d\sqrt T$.}
If instead $\log K\le d$, the case assumption forces $M\log(2KT+1)>T/(16\min\{K,d+1\})$.
The deterministic bound $R_T\le2T$, together with this case condition, gives
 {$\EE[R_T]\le2T<8\sqrt{\min\{K,d+1\}TM\log(2KT+1)} \le32\sqrt{\min\{K,d+1\}TM\log(2KT+1)}$.}
Thus the right-hand side of \eqref{eq:geo-general} also bounds this case, without applying
Lemma~\ref{lem:geo} outside its learning-rate range. The same substitution as in Case 1 yields
 {$\EE[R_T]\le167\log(2dT)\,d^{3/4}\sqrt T\,(\log K)^{1/4}$.}
Combining the two cases gives
 {$\EE[R_T]\le167\log(2dT)\sqrt{dT}\min\left\{\sqrt d,(d\log K)^{1/4}\right\}$,}
which proves \eqref{eq:large-upper} with $C=167$.
The constant calculation is given in Appendix~\ref{app:constants}.
\end{proof}

\section{Proofs of the Upper Bound Lemmas}\label{app:upper}
This appendix proves the repair, master, and surrogate lemmas used in Section~\ref{sec:upper-analysis}. The parameter and the menu mechanism
are fixed throughout. At every reward update, conditional on the current menu, the learner's action,
the adversary's private state, and the current repair sequences, the new noise remains independent
$\N(0,1)$. In particular, no argument requires the current menu to be independent of past rewards.

\subsection{Proof of Lemma~\ref{lem:budget}}\label{app:budget}
The proof shares several ingredients with the standard OFUL analysis
\citep{abbasi2011}, in particular the self-normalized prediction-error
potential and the control of information growth through
$\log\det V_t$. The main difference is that our predictor is additionally
modified by repair updates interleaved with the usual ridge-regression
updates, and we must control the total number of such repairs.
\begin{proof}[Proof of Lemma~\ref{lem:budget}]
For the proof, suppose that the true parameter $\theta$ is known, and use it
 {only to construct one repair sequence. We allow this auxiliary sequence to
grow without imposing the budget $M$, and then bound its length.
Starting from the empty sequence, at}
round $t$ let
$v=V_t^{-1}(b_t+c_e)$
be the predictor associated with the repairs constructed so far. While there
exists an action $i\in[K]$ such that
$|x_{t,i}^\top(\theta-v)|>\alpha w_{t,i}$,
set
 {$i$ to the smallest violating index and
$s=\operatorname{sign}\!\left(x_{t,i}^\top(\theta-v)\right)$,
and append $(t,i,s)$ to the sequence.
Such an index has $w_{t,i}>0$, so the update is well defined.} Applying this repair changes the
predictor to
 {$v' = v+\frac{\alpha s}{w_{t,i}}V_t^{-1}x_{t,i}$.}
We continue repairing until no displayed action violates the desired bound.
Thus, if $v_t^+$ denotes the predictor after all repairs in round $t$, then
\begin{equation}\label{eq:repair-stopping}
|x_{t,i}^\top(\theta-v_t^+)|
\le
\alpha w_{t,i}
\qquad
\text{for every }i\in[K].
\end{equation}
We first show that each repair makes definite progress. Consider a repair on
$x_{t,i}$ and write
 {$\delta=x_{t,i}^\top(\theta-v)$.}
The repair is made only when $|\delta|>\alpha w_{t,i}$, and
$s=\operatorname{sign}(\delta)$. Since
$w_{t,i}^2=x_{t,i}^\top V_t^{-1}x_{t,i}$,

\begin{equation}\label{eq:one-repair-progress}
\|\theta-v'\|_{V_t}^2
=\|\theta-v\|_{V_t}^2-\frac{2\alpha}{w_{t,i}}|\delta|+\alpha^2
<\|\theta-v\|_{V_t}^2-\alpha^2.
\end{equation}

 {By \eqref{eq:one-repair-progress}, every repair decreases the squared $V_t$-error by more than}
$\alpha^2$. In particular, since $V_t$ is fixed during the repair loop and
the squared error is nonnegative, only finitely many repairs can occur in
each round.
We now bound the total number of repairs over all rounds. Let $v_t^-$ and
$v_t^+$ denote the predictors immediately before and after the repairs in
round $t$, respectively, and let $C_t$ be the total number of repairs through
round $t$, with $C_0=0$. Define
 {$U_t^-:=\|\theta-v_t^-\|_{V_t}^2, \qquad U_t^+:=\|\theta-v_t^+\|_{V_t}^2$.}
There are $C_t-C_{t-1}$ repairs in round $t$, and each decreases the squared
error by more than $\alpha^2$. Therefore
\begin{equation}\label{eq:repair-drop}
U_t^+
\le
U_t^--\alpha^2(C_t-C_{t-1}),
\end{equation}
or equivalently,
 {$U_t^++\alpha^2C_t \le U_t^-+\alpha^2C_{t-1}$.}
It remains to control how much the noisy reward update can increase this
potential. After the repairs, the learner chooses $x_t=x_{t,A_t}$ and observes
 {$Y_t=x_t^\top\theta+\varepsilon_t$.}

The oracle offset does not change during the reward update.
Since $V_tv_t^+=b_t+c_e$, we have
$V_{t+1}v_t^++x_t(Y_t-x_t^\top v_t^+)=b_{t+1}+c_e$.
Multiplication by $V_{t+1}^{-1}$ gives
\begin{equation}\label{eq:repair-ridge-update}
v_{t+1}^-=v_t^++V_{t+1}^{-1}x_t(Y_t-x_t^\top v_t^+).
\end{equation}

Set $u:=x_t^\top(\theta-v_t^+)$ and $r:=x_t^\top V_t^{-1}x_t$.
Then $Y_t-x_t^\top v_t^+=u+\varepsilon_t$.
The Sherman--Morrison identity gives
$V_{t+1}^{-1}x_t=V_t^{-1}x_t/(1+r)$ and
$x_t^\top V_{t+1}^{-1}x_t=r/(1+r)$.
Using these identities in \eqref{eq:repair-ridge-update}, with
$q:=\theta-v_t^+$, yields
\begin{equation}\label{eq:repair-potential-step}
\begin{aligned}
U_{t+1}^-&=q^\top V_{t+1}q-2(u+\varepsilon_t)q^\top x_t
 +(u+\varepsilon_t)^2\frac r{1+r}\\
&=U_t^++\frac{-u^2-2u\varepsilon_t+r\varepsilon_t^2}{1+r}.
\end{aligned}
\end{equation}

Conditional on the history and the chosen action, $u,r$ are fixed and
$\varepsilon_t\sim\N(0,1)$. Completing the square in the Gaussian density
shows that
\begin{equation}\label{eq:repair-gaussian-mgf}
\begin{aligned}
\EE_t\left[\left.
\exp\!\left\{\frac{-u^2-2u\varepsilon_t+r\varepsilon_t^2}{2(1+r)}\right\}
\right|A_t\right]
&=\frac1{\sqrt{2\pi}}\int_{\RR}
\exp\!\left\{-\frac{(z+u)^2}{2(1+r)}\right\}\,dz\\
&=\sqrt{1+r}.
\end{aligned}
\end{equation}
Averaging this conditional identity over $A_t$ gives the next
supermartingale inequality.

Define $S_0:=\exp(\|\theta\|_2^2/2)$ and
$S_t:=\exp((U_{t+1}^-+\alpha^2C_t)/2)/\sqrt{\det V_{t+1}}$ for
$t\in[T]$. The determinant identity is
$\det V_{t+1}=\det V_t(1+r)$.
Combining \eqref{eq:repair-drop}, \eqref{eq:repair-potential-step}, and
\eqref{eq:repair-gaussian-mgf}, conditional first on $A_t$, gives
\begin{equation}\label{eq:repair-supermartingale}
\EE_t[S_t\mid A_t]
\le\frac{\exp((U_t^-+\alpha^2C_{t-1})/2)}{\sqrt{\det V_t}}.
\end{equation}
For $t\ge2$, the right-hand side equals $S_{t-1}$; for $t=1$, it equals
$S_0$ because $V_1=I_d$, $v_1^-=0$, and $C_0=0$.
It does not depend on the sampled action. Averaging
\eqref{eq:repair-supermartingale} over that action and then taking conditional
expectations before the current menu is chosen proves that $(S_t)_{t=0}^T$ is a nonnegative supermartingale.
In particular, $\EE S_T\le S_0\le e^{1/2}$.

Since $\tr(V_{T+1})=d+\sum_{t=1}^T\|x_t\|_2^2\le d+T$,
the arithmetic--geometric mean inequality for its eigenvalues gives
$\det V_{T+1}\le(1+T/d)^d$.
Together with $U_{T+1}^-\ge0$, this implies
$S_T\ge\exp((\alpha^2C_T-d\log(1+T/d))/2)$.
Hence $\alpha^2C_T>1+d\log(1+T/d)+4\log T$ implies $S_T>e^{1/2}T^2$.
Markov's inequality yields
\begin{equation}\label{eq:repair-count-tail}
\PP\{\alpha^2C_T>1+d\log(1+T/d)+4\log T\}
\le\frac{\EE S_T}{e^{1/2}T^2}\le T^{-2}.
\end{equation}
By \eqref{eq:repair-count-tail}, with probability at least $1-T^{-2}$ we have
$C_T\le(1+d\log(1+T/d)+4\log T)/\alpha^2\le M$.
 On this event, the repairs constructed above, recorded in execution order,
form a sequence $e^\star\in\cE_M$. By construction, after all repairs in
round $t$ its predictor is exactly $v_{t,e^\star}=v_t^+$. Therefore
\eqref{eq:repair-stopping} gives
 {$|x_{t,i}^\top(\theta-v_{t,e^\star})| \le \alpha w_{t,i} \qquad \text{for every }t\in[T]\text{ and }i\in[K]$.}
\end{proof}
\subsection{Proof of Lemma~\ref{lem:ix}}\label{app:ix}
The proof of this lemma is a direct adaptation of the standard EXP4-IX analysis of
\citet{neu2015}.
\begin{proof}
Condition on the history, current menu, and all repair-expert recommendations before
sampling $A_t$, and denote the corresponding conditional expectation by $\EE_t$. For this proof only,
write $q_{t,e}:=W_e/\sum_{f\in\cE_M}W_f$ and define $\widehat\ell_t(i):=Z_t\1_{\{A_t=i\}}/(p_{t,i}+\eta)$. Then $\widehat\ell_{t,e}=\widehat\ell_t(a_{t,e})$ and
$p_{t,i}=\sum_{e\in\cE_M}q_{t,e}\1_{\{a_{t,e}=i\}}$.
Using $e^{-u}\le1-u+u^2/2$ for $u\ge0$ and $\log(1+u)\le u$, the exponential-weights update gives
 {$\log\frac{\sum_eW_{t+1,e}}{\sum_eW_{t,e}} \le-\eta\sum_eq_{t,e}\widehat\ell_{t,e}+\frac{\eta^2}2\sum_eq_{t,e}\widehat\ell_{t,e}^2$.}
Summing over time and comparing with any fixed repair expert yields

\begin{equation}\label{eq:ix-pathwise}
\sum_t\langle p_t,\widehat\ell_t\rangle-\sum_t\widehat\ell_t(a_{t,e})
\le\frac{\log|\cE_M|}{\eta}+\frac\eta2\sum_t\sum_ip_{t,i}\widehat\ell_t(i)^2.
\end{equation}

Because $0\le Z_t\le1$,
 {$\EE_t\sum_ip_{t,i}\widehat\ell_t(i)^2 =\sum_i\frac{p_{t,i}^2\EE_t[Z_t^2\mid A_t=i]}{(p_{t,i}+\eta)^2}\le K$.}
Moreover, since $\EE_t[Z_t\mid A_t=i]=\bar\ell_t(i)$,
 {$\langle p_t,\bar\ell_t\rangle-\EE_t\langle p_t,\widehat\ell_t\rangle =\sum_i\frac{\eta p_{t,i}\bar\ell_t(i)}{p_{t,i}+\eta}\le\eta K$.}
It remains to allow the comparator to be selected after the trajectory is known. Fix $e\in\cE_M$.
Since $0\le\eta\widehat\ell_t(a_{t,e})\le1$, $e^u\le1+u+u^2$ on $[0,1]$, and $Z_t^2\le Z_t$,

\begin{equation}\label{eq:ix-one-sided-mgf}
\EE_te^{\eta\widehat\ell_t(a_{t,e})}
\le1+\frac{p_{t,a_{t,e}}\eta\bar\ell_t(a_{t,e})}{p_{t,a_{t,e}}+\eta}
+\frac{p_{t,a_{t,e}}\eta^2\bar\ell_t(a_{t,e})}{(p_{t,a_{t,e}}+\eta)^2}
\le e^{\eta\bar\ell_t(a_{t,e})}.
\end{equation}

 {Iterating \eqref{eq:ix-one-sided-mgf} gives}
 {$\EE\exp\left\{\eta\sum_t\bigl(\widehat\ell_t(a_{t,e})-\bar\ell_t(a_{t,e})\bigr)\right\}\le1$.}
Applying log-sum-exp and Jensen's inequality over all $e\in\cE_M$ yields

\begin{equation}\label{eq:ix-comparator-error}
\EE\max_{e\in\cE_M}\sum_t\bigl(\widehat\ell_t(a_{t,e})-\bar\ell_t(a_{t,e})\bigr)
\le\frac{\log|\cE_M|}{\eta}.
\end{equation}

Choosing the repair expert minimizing $\sum_t\bar\ell_t(a_{t,e})$ on the realized trajectory and combining
 {\eqref{eq:ix-pathwise}},  {\eqref{eq:ix-comparator-error}, and the moment and bias bounds gives}
 {$\EE\left[\sum_t\bar\ell_t(A_t)-\min_{e\in\cE_M}\sum_t\bar\ell_t(a_{t,e})\right] \le\frac{2\log|\cE_M|}{\eta}+\frac32\eta KT$,}
which proves \eqref{eq:ix-master-bound}.
\end{proof}

\subsection{Proof of Lemma~\ref{lem:bridge}}\label{app:surrogate}
\begin{proof}
Work on the event of Lemma~\ref{lem:budget} and write $a:=a_t^\star$. By \eqref{eq:good-repair},
 {$\mu_{t,i_t^\star}\le x_{t,i_t^\star}^\top v_{t,e^\star}+\alpha w_{t,i_t^\star}$.}
Since $a$ maximizes the optimistic score in \eqref{eq:recommendation},
 {$x_{t,i_t^\star}^\top v_{t,e^\star}+\alpha w_{t,i_t^\star} \le x_{t,a}^\top v_{t,e^\star}+\alpha w_{t,a}$.}
Applying \eqref{eq:good-repair} once more gives
 {$\mu_{t,i_t^\star}-\mu_{t,a}\le2\alpha w_{t,a}$.}
For Gaussian rewards,
$\EE[\1_{\{Y_t\le0\}}\mid H_{t-1},\cX_t,A_t=i]=\Phi(-\mu_{t,i})=\ell_t(i)$.
Thus the discounted observation used by Algorithm~\ref{alg:upper} has
conditional mean $\bar\ell_t(i)$. Since the derivative of $\Phi$ is bounded
by $\phi(0)<1/2$,
 {$0\le\ell_t(a)-\ell_t(i_t^\star)\le\phi(0)\bigl(\mu_{t,i_t^\star}-\mu_{t,a}\bigr)\le\alpha w_{t,a}$.}
This difference is also at most one, so it is at most $\beta_{t,a}$. Therefore
 {$\bar\ell_t(a)=\frac{\ell_t(a)+1-\beta_{t,a}}2\le\frac{\ell_t(i_t^\star)+1}2$,}
which proves \eqref{eq:domination}. For the geometric master, the reward gap is bounded by both $2\alpha w_{t,a}$ and $2$, and hence by $2\beta_{t,a}$.
Thus
 {$g_t(a)=\frac{\mu_{t,a}+2\beta_{t,a}-1}2\ge\frac{\mu_{t,i_t^\star}-1}2$,}
which proves \eqref{eq:geo-domination}. Finally, let $r_t:=w_{t,A_t}^2=x_t^\top V_t^{-1}x_t$. Since $0\le r_t\le1$ and $r\le2\log(1+r)$ on $[0,1]$,
 {$\sum_{t=1}^T w_{t,A_t}^2\le2\sum_{t=1}^T\log(1+r_t)=2\log\det V_{T+1}\le2d\log(1+T/d)$.}
Since $\beta_{t,A_t}\le\alpha w_{t,A_t}$, Cauchy--Schwarz gives
 {$\sum_{t=1}^T\beta_{t,A_t}\le\alpha\sqrt{2Td\log(1+T/d)}$,}
which proves \eqref{eq:played-uncertainty}.
\end{proof}

\subsection{Proof of Lemma~\ref{lem:geo}}\label{app:geometry}
We first establish the finite-design property used by the geometric master.

\subsubsection{A design on the current feature span}\label{app:geo-design}

The design condition below is a constant-factor relaxation of the usual
D-optimal design condition. We use a vertex-direction method for the log
determinant objective, as in algorithms for enclosing ellipsoids
\citep{todd2007}; the fixed-step version needed here has a short proof.

\begin{definition}[Approximate exploration design]\label{def:design}
Let $z_1,\ldots,z_K$ span a subspace of dimension $s\ge1$.
A distribution $\nu\in\Delta_K$ is an approximate exploration design if
$J:=\sum_i\nu_i z_i z_i^\top$ is positive definite on this span and
$\max_i z_i^\top J^\dagger z_i\le2s$.
\end{definition}

\begin{lemma}[Finite-design leverage bound]\label{lem:design}
For any such vectors, the following procedure returns an approximate
exploration design. Start from $\nu_i=1/K$. At each step, form
$J=\sum_i\nu_i z_i z_i^\top$ and let $j$ be the smallest index maximizing
$z_j^\top J^\dagger z_j$. Stop if this maximum is at most $2s$; otherwise set
$\nu\leftarrow(1-1/(2s))\nu+(1/(2s))e_j$, where $e_j$ is the $j$th
coordinate vector in $\RR^K$.
The procedure stops after at most $\lceil20s\log K\rceil$ updates and gives
\begin{equation}\label{eq:design-leverage}
z_i^\top J^\dagger z_i\le2s\quad\text{for every }i\in[K].
\end{equation}
Its output is a Borel function of the ordered vectors.
\end{lemma}
\begin{proof}
Work in orthonormal coordinates on the span. The uniform second-moment
matrix $J_0$ is positive definite, and every update retains this property.
For every distribution $\nu$, $J(\nu)\preceq KJ_0$, so
$\log\det J(\nu)-\log\det J_0\le s\log K$.
If an update is made, let $h:=z_j^\top J^{-1}z_j>2s$ and
$\tau:=1/(2s)$. The determinant lemma gives
\begin{equation}\label{eq:design-progress}
\frac{\det((1-\tau)J+\tau z_jz_j^\top)}{\det J}
=(1-\tau)^{s-1}(1-\tau+\tau h)
>(1-1/(2s))^{s-1}(2-1/(2s)).
\end{equation}
For $s=1$, the last term in \eqref{eq:design-progress} is $3/2$. For
$s\ge2$, its logarithm is
at least $-1/2+\log(7/4)>1/20$, since
$\log(1-u)\ge-u/(1-u)$ for $0<u<1$.
Thus each update increases the log determinant by more than $1/20$.
The total increase is at most $s\log K$, so the procedure must stop within
the stated number of updates. Its stopping condition is
\eqref{eq:design-leverage}.

Rank and the Moore--Penrose inverse are Borel functions of a matrix.
Choosing the smallest maximizing index and performing each update are
also Borel operations. The bounded number of updates therefore makes the
returned distribution Borel measurable.
\end{proof}
An orthonormal basis of the span and the reduced coordinates can be formed
in $O(K(d+1)s)$ arithmetic operations by incremental orthogonalization.
In rank-$s$ coordinates, each step evaluates the $K$ quadratic forms and
updates the inverse by a rank-one formula, using $O(Ks^2)$ operations.
Thus the total cost is $O(K(d+1)s+Ks^3\log K)$ real-arithmetic operations
for an explicitly listed menu.

\subsubsection{Conditional moments and the master comparison}\label{app:geo-moments}
\begin{proof}[Proof of Lemma~\ref{lem:geo}]
Fix a round and suppress the time index. Condition before sampling the
current action, so the lifted features, repair-expert recommendations, and weights are fixed. Let
their span have dimension $1\le s\le\min\{K,d+1\}$.

For this proof only, let $q_e:=W_e/\sum_fW_f$ and $J:=\sum_i\nu_i z_iz_i^\top$. Since $\Gamma\succeq\gamma J$, Lemma~\ref{lem:design} gives
 {$h_e=z_{a_e}^\top\Gamma^\dagger z_{a_e}\le\frac{2\min\{K,d+1\}}\gamma$.}
Also, $\Gamma\succeq(1-\gamma)\sum_ip_i^0z_iz_i^\top$, and therefore

\begin{equation}\label{eq:geo-average-leverage}
\sum_eq_eh_e=\tr\left(\Gamma^\dagger\sum_ip_i^0z_iz_i^\top\right)
\le\frac s{1-\gamma}\le2\min\{K,d+1\}.
\end{equation}

For a fixed repair expert, define $a_i:=z_{a_e}^\top\Gamma^\dagger z_i$. The inverse-metric Cauchy--Schwarz inequality
and $\Gamma=\sum_ip_iz_iz_i^\top$ give

\begin{equation}\label{eq:geo-coefficients}
|a_i|\le\frac{2\min\{K,d+1\}}\gamma,\qquad\sum_ip_ia_i^2=h_e,\qquad\sum_ip_ia_ig(i)=g(a_e).
\end{equation}

Conditionally on $A=i$, write $Z=g(i)+\xi_i$, where $\xi_i$ is centered and conditionally $1$-subGaussian.
For $2|\lambda|\min\{K,d+1\}/\gamma\le1/4$, the geometric estimate $\widehat g_e=a_AZ$ satisfies
\begin{equation}\label{eq:geo-mgf}
\EE_te^{\lambda\widehat g_e}\le1+\lambda g(a_e)+2\lambda^2h_e\le e^{\lambda g(a_e)+2\lambda^2h_e}.
\end{equation}
Indeed, with $u=\lambda a_i$, $\EE_t[e^{uZ}\mid A=i]\le e^{ug(i)+u^2/2}$. Under the stated range of $\lambda$, the exponent has
absolute value at most $9/32$, so
 {$e^{ug(i)+u^2/2}\le1+ug(i)+2u^2$.}
 {Averaging over $A$ and using \eqref{eq:geo-coefficients} gives \eqref{eq:geo-mgf}.}

Return to the time-indexed notation and define locally $\widetilde g_{t,e}:=\widehat g_{t,e}+2\eta h_{t,e}$. Since $\gamma=8\min\{K,d+1\}\eta$ and
$\eta\le1/(16\min\{K,d+1\})$, the negative-moment side of \eqref{eq:geo-mgf} gives
 {$\EE_t\exp\left\{\eta\bigl(g_t(a_{t,e})-\widetilde g_{t,e}\bigr)\right\}\le1$.}
Iterated conditioning and log-sum-exp therefore imply

\begin{equation}\label{eq:geo-comparator-error}
\EE\max_{e\in\cE_M}\sum_t\bigl(g_t(a_{t,e})-\widetilde g_{t,e}\bigr)\le\frac{\log|\cE_M|}{\eta}.
\end{equation}

For the weight potential, let $b_e:=2\eta^2h_{t,e}$ for this calculation.
The leverage bound gives $0\le b_e\le1/16$, and
$|\eta g_t(a_{t,e})|\le1/8$. The positive-moment part of
\eqref{eq:geo-mgf} and $e^{b_e}\le1+2b_e$ imply
\begin{equation}\label{eq:geo-weight-moment}
\begin{aligned}
\EE_t e^{\eta\widehat g_{t,e}+b_e}
&\le e^{b_e}\bigl(1+\eta g_t(a_{t,e})+b_e\bigr)\\
&\le1+\eta g_t(a_{t,e})+4b_e
=1+\eta g_t(a_{t,e})+8\eta^2h_{t,e}.
\end{aligned}
\end{equation}
Here the second inequality follows by expanding the product with
$1+2b_e$ and using the two stated bounds. Apply Jensen's inequality to
the logarithm, average \eqref{eq:geo-weight-moment} with weights $q_{t,e}$,
 {and use $\log(1+x)\le x$ and \eqref{eq:geo-average-leverage}.} This gives
\begin{equation}\label{eq:geo-weight-potential}
\EE_t\log\frac{\sum_eW_{t+1,e}}{\sum_eW_{t,e}}
\le\eta\sum_eq_{t,e}g_t(a_{t,e})+16\eta^2\min\{K,d+1\}.
\end{equation}
 On the other hand,

\begin{equation}\label{eq:geo-terminal-potential}
\log\frac{\sum_eW_{T+1,e}}{|\cE_M|}\ge\eta\max_{e\in\cE_M}\sum_t\widetilde g_{t,e}-\log|\cE_M|.
\end{equation}

 {Combining \eqref{eq:geo-comparator-error}, \eqref{eq:geo-weight-potential}, and
\eqref{eq:geo-terminal-potential} gives comparison with the repair-expert mixture of at most}
 {$\frac{2\log|\cE_M|}{\eta}+16\eta\min\{K,d+1\}T$.}
Finally, mixing the repair-expert distribution with $\nu_t$ changes a mean in $[-1,1]$ by at most $2\gamma$
per round. Since $\gamma=8\min\{K,d+1\}\eta$, the comparison with the learner's actual actions is at most
 {$\frac{2\log|\cE_M|}{\eta}+32\eta\min\{K,d+1\}T$,}
which proves \eqref{eq:geo-master-bound}.
\end{proof}

\subsection{Elementary Bounds Used in the Theorem Proofs}\label{app:constants}
The calculations below give $C=130$ in Theorem~\ref{thm:upper}
and $C=167$ in Theorem~\ref{thm:large}.
For $T\ge d\ge2$, Bernoulli's inequality gives $(1+T/d)^d\ge1+T$. Hence $d\log(1+T/d)\ge\log(1+T)>1$,
$d\log(1+T/d)\ge d\log2$, and
 {$1+d\log(1+T/d)+4\log T\le6d\log(1+T/d)$.}
When $K\le d\le T$, $2KT+1\le(1+T)^3$, so
 {$\log(2KT+1)\le3d\log(1+T/d)$.}
Together with $K\le d\le d\log(1+T/d)/\log2$, this gives
 {$K\log(2KT+1)\le\frac{3[d\log(1+T/d)]^2}{\log2}$.}
Since $M\le6d\log(1+T/d)/\alpha^2+1$ and $|\cE_M|\le(2KT+1)^M$, the choice $\eta=\sqrt{M\log(2KT+1)/(KT)}$
gives
 {$\frac{4\log|\cE_M|}{\eta}+3\eta KT\le7\sqrt{KTM\log(2KT+1)}$.}
With $\alpha=[K\log(2KT+1)]^{1/4}$, the right-hand side is at most
 {$7\left[\sqrt6+(3/\log2)^{1/4}\right][K\log(2KT+1)]^{1/4}\sqrt{Td\log(1+T/d)}$.}
Adding the played-uncertainty and failure terms and using $\phi(1)>0.241$ gives a coefficient below
$124<130$.

For the geometric branch, take $\alpha=[\min\{K,d+1\}\log(2KT+1)]^{1/4}$. The general
geometric bound gives

\begin{equation}\label{eq:geo-constant-reduction}
\begin{aligned}
\EE[R_T]\le{}&84[\min\{K,d+1\}\log(2KT+1)]^{1/4}\sqrt{Td\log(1+T/d)}\\
&+32\sqrt{\min\{K,d+1\}T\log(2KT+1)}.
\end{aligned}
\end{equation}

When $K\ge d\ge2$ and $T\ge d$, we have $\min\{K,d+1\}\le3d/2$, $d\log(1+T/d)\le d\log(2dT)$, and
 {$\log(2KT+1)\le\log K+\log(2dT)\le2\log K\log(2dT)$.}
 {If $\log K\le d$, the first term in \eqref{eq:geo-constant-reduction} is at most
$111\log(2dT)d^{3/4}\sqrt T(\log K)^{1/4}$.
The second satisfies the same bound with coefficient $56$.
Their sum gives the coefficient $167$.}
For OFUL, using $\log T\le d\log(1+T/d)$,
$d\log(1+T/d)>1$, and $2/T\le d\log(1+T/d)\sqrt T$ in
\eqref{eq:oful-bound} gives
 {$\EE[R_T]\le(5+4\sqrt5)d\log(1+T/d)\sqrt T <16d\log(1+T/d)\sqrt T$.}
In particular, if $\log K>d$, then
 {$\EE[R_T]\le16\log(2dT)\,d\sqrt T$.}
Together with the deterministic bound $R_T\le2T$ used in Case 2 of the
proof of Theorem~\ref{thm:large}, these estimates cover both branches and
give \eqref{eq:large-upper} with $C=\max\{167,16\}=167$.

\section{Noise Extensions}\label{app:noise}
The repair potential and the geometric master require only conditional subGaussian moments. The
small-menu algorithm is different because its bounded loss observation $\1_{\{Y_t\le0\}}$ uses the Gaussian
 {identity $\EE_t[\1_{\{Y_t\le0\}}\mid A_t=i]=\Phi(-\mu_{t,i})$.} We therefore treat the general subGaussian extension and
the bounded-reward small-menu extension separately.

\begin{lemma}\label{lem:subgaussian-potential}
Condition on the history, current menu, and chosen
action. Suppose the centered noise satisfies $\EE e^{\lambda\varepsilon}\le e^{\lambda^2/2}$ for every $\lambda\in\RR$. Then, for every $u\in\RR$
and $r\ge0$,
\begin{equation}\label{eq:sg-mgf}
\EE\exp\left\{\frac{-u^2-2u\varepsilon+r\varepsilon^2}{2(1+r)}\right\}\le\sqrt{1+r}.
\end{equation}
Consequently, the good-repair guarantee of Lemma~\ref{lem:budget} remains valid with the same constants.
\end{lemma}
\begin{proof}
All expectations may be read conditionally. Introduce an independent $G\sim\N(0,1)$. For
$q\in[0,1)$ and $a\in\RR$, Tonelli's theorem and the subGaussian moment bound give

\begin{equation}\label{eq:gaussian-linearization}
\begin{aligned}
\EE_\varepsilon e^{q\varepsilon^2/2-a\varepsilon}
&=\EE_G\EE_\varepsilon e^{(\sqrt q\,G-a)\varepsilon}
\le\EE_Ge^{(\sqrt q\,G-a)^2/2}\\
&=(1-q)^{-1/2}\exp\left\{\frac{a^2}{2(1-q)}\right\}.
\end{aligned}
\end{equation}

 {In \eqref{eq:gaussian-linearization}, set $q=r/(1+r)$ and $a=u/(1+r)$ and multiply} by $e^{-u^2/(2(1+r))}$. The terms involving $u$ cancel
and give \eqref{eq:sg-mgf}.  {Replacing \eqref{eq:repair-gaussian-mgf} by this inequality} leaves the supermartingale argument
unchanged.
\end{proof}

\begin{corollary}[subGaussian upper bounds]\label{cor:subgaussian}
Suppose $Y_t$ has conditional mean $x_t^\top\theta$ and conditionally
$1$-subGaussian centered noise after the current menu and action are fixed. Then the choice between
Algorithm~\ref{alg:geo} and OFUL with $\lambda=1$ and $\delta=T^{-2}$,
as specified in Theorem~\ref{thm:large}, satisfies the same bound.
Moreover, for $2\le K\le d\le T$, $\min\{K,d+1\}=K$ and $\alpha=[K\log(2KT+1)]^{1/4}$.
Run Repair-Geo with $\eta=\sqrt{M\log(2KT+1)/(16KT)}$ when $M\log(2KT+1)\le T/(16K)$,
and otherwise run OFUL with $\lambda=1$ and $\delta=T^{-2}$ as in
Lemma~\ref{lem:oful}. This choice gives
$\EE R_T\le130[K\log(2KT+1)]^{1/4}\sqrt{dT\log(1+T/d)}$.
\end{corollary}
\begin{proof}
Lemma~\ref{lem:subgaussian-potential} gives the same good repair sequence as before. The lifted observation has centered
noise $\varepsilon_t/2$, so it satisfies the assumption of Lemma~\ref{lem:geo}. Its feature rank is at most $\min\{K,d+1\}$.
Hence the geometric-branch proof of Theorem~\ref{thm:large} applies without
change, while the OFUL guarantee of Lemma~\ref{lem:oful} already holds for
conditionally $1$-subGaussian centered noise.

For $K\le d$, $\min\{K,d+1\}=K$. When $M\log(2KT+1)\le T/(16K)$, the geometric bound
\eqref{eq:geo-general} applies. Otherwise, $R_T\le2T$ and
$M\log(2KT+1)>T/(16K)$ give
 {$\EE[R_T]\le2T<8\sqrt{KTM\log(2KT+1)} \le32\sqrt{KTM\log(2KT+1)}$.}
Thus the right-hand side of \eqref{eq:geo-general} still applies.

For the constant, apply $M\le6d\log(1+T/d)/\alpha^2+1$ directly in
\eqref{eq:geo-general}, before rounding its coefficients.
After division by $\alpha\sqrt{Td\log(1+T/d)}$, the resulting bound is at most
$32\sqrt6+2\sqrt2+32(3/\log2)^{1/4}+1<130$.
Here the third term uses
$K\log(2KT+1)\le3[d\log(1+T/d)]^2/\log2$, and the last term bounds
$2/T$. This proves the stated constant.

\end{proof}

\begin{corollary}[Bounded rewards and the IX implementation]\label{cor:bounded}
Suppose instead that $Y_t\in[-1,1]$
almost surely with conditional mean $x_t^\top\theta$. Modify Algorithm~\ref{alg:upper} only by replacing its loss observation
with
\begin{equation}\label{eq:bounded-statistic}
Z_t:=\frac{(1-Y_t)/2+1-\beta_{t,A_t}}2.
\end{equation}
Then, for every $\alpha>0$ and $\eta>0$,
\begin{equation}\label{eq:bounded-general}
\EE R_T\le2\left(\frac{4\log|\cE_M|}{\eta}+3\eta KT+\alpha\sqrt{2Td\log(1+T/d)}+\frac2T\right).
\end{equation}
In particular, the tuning in Theorem~\ref{thm:upper} gives the same bound \eqref{eq:upper-main}. No bounded-reward lower bound
is asserted.
\end{corollary}
\begin{proof}
The centered reward noise has support in an interval of length two and is therefore $1$-subGaussian.
Hence Lemma~\ref{lem:subgaussian-potential} gives the same good repair sequence.

For this setting, use the bounded loss $\ell_t(i):=(1-\mu_{t,i})/2$ and $\bar\ell_t(i):=(\ell_t(i)+1-\beta_{t,i})/2$. The
statistic in \eqref{eq:bounded-statistic} lies in $[0,1]$ and has conditional mean $\bar\ell_t(i)$ when action $i$ is played.

On the event of Lemma~\ref{lem:budget}, the good repair expert satisfies $\mu_{t,i_t^\star}-\mu_{t,a_t^\star}\le2\alpha w_{t,a_t^\star}$, and therefore
 {$\ell_t(a_t^\star)-\ell_t(i_t^\star)\le\alpha w_{t,a_t^\star}$.
The loss gap is also at most one, so it is at most
$\min\{1,\alpha w_{t,a_t^\star}\}=\beta_{t,a_t^\star}$.}
Thus the same discounted-loss argument gives $\bar\ell_t(a_t^\star)\le(\ell_t(i_t^\star)+1)/2$.

Applying exactly the same decomposition as in the proof of Theorem~\ref{thm:upper}, Lemma~\ref{lem:ix} controls
the repair-expert comparison and \eqref{eq:played-uncertainty} controls the remaining played uncertainty. Since reward
regret is exactly twice the surrogate-loss regret, we obtain \eqref{eq:bounded-general}. Substituting the parameters from
Theorem~\ref{thm:upper} gives \eqref{eq:upper-main}.
\end{proof}

\end{document}